\documentclass{article} 
\usepackage{iclr2026_conference,times}

\usepackage{amsmath,amsfonts,bm}

\def\eqref#1{equation~\ref{#1}}

\def\1{\bm{1}}

\DeclareMathAlphabet{\mathsfit}{\encodingdefault}{\sfdefault}{m}{sl}
\SetMathAlphabet{\mathsfit}{bold}{\encodingdefault}{\sfdefault}{bx}{n}

\usepackage{subcaption}  
\usepackage{hyperref}
\usepackage{url}
\usepackage{algpseudocode}
\usepackage{algorithm}
\usepackage{amsthm}  
\newtheorem{proposition}{Proposition}
\usepackage{subcaption}  
\usepackage{wrapfig}      
\usepackage{tcolorbox}    
\tcbuselibrary{listings}  
\usepackage{amsmath}      

\usepackage{graphicx}
\usepackage{colortbl}
\usepackage{graphicx}
\usepackage{booktabs}
\usepackage{multirow}
\usepackage{wrapfig}

\title{LaplacianFormer:Rethinking Linear Attention with Laplacian Kernel}

\author{
Zhe Feng\textsuperscript{2,1} \thanks{Equal contribution.}\quad
Sen Lian\textsuperscript{3} \footnotemark[1]\quad
Changwei Wang\textsuperscript{5,6}\footnotemark[2]\enspace
Muyang Zhang\textsuperscript{2,1}\enspace
Tianlong Tan\textsuperscript{4}\enspace
Rongtao Xu\textsuperscript{7}\enspace
\\
\bfseries 
Weiliang Meng\textsuperscript{1,2}\footnotemark[2]\enspace
Xiaopeng Zhang\textsuperscript{1,2} \\
\textsuperscript{1}MAIS,Institute of Automation, Chinese Academy of Sciences\\
\textsuperscript{2}School of Artificial Intelligence, University of Chinese Academy of Sciences\\usr 
\textsuperscript{3}China Electronics Data Corporation\\
\textsuperscript{4}Institute of Computing Technology, Chinese Academy of Sciences \\
\textsuperscript{5}The Key Laboratory of Computing Power Network and Information Security, \\Ministry of Education, Shandong Computer Science Center, Qilu University of Technology\\
\textsuperscript{6}Shandong Provincial Key Laboratory of Computing Power Internet and Service Computing, \\Shandong Fundamental Research Center for Computer Science\\
\textsuperscript{7} Spatialtemporal AI\\
\texttt{changweiwang@sdas.org, weiliang.meng@ia.ac.cn}
}

\iclrfinalcopy 
\begin{document}

\maketitle
\renewcommand{\thefootnote}{\fnsymbol{footnote}}
\footnotetext[2]{Corresponding author.}
\renewcommand*{\thefootnote}


\begin{abstract}
The quadratic complexity of softmax attention presents a major obstacle for scaling Transformers to high-resolution vision tasks. Existing linear attention variants often replace the softmax with Gaussian kernels to reduce complexity, but such approximations lack theoretical grounding and tend to oversuppress mid-range token interactions. We propose LaplacianFormer, a Transformer variant that employs a Laplacian kernel as a principled alternative to softmax, motivated by empirical observations and theoretical analysis. To address expressiveness degradation under low-rank approximations, we introduce a provably injective feature map that retains fine-grained token information. For efficient computation, we adopt a Nyström approximation of the kernel matrix and solve the resulting system using Newton--Schulz iteration, avoiding costly matrix inversion and SVD. We further develop custom CUDA implementations for both the kernel and solver, enabling high-throughput forward and backward passes suitable for edge deployment. Experiments on ImageNet show that LaplacianFormer achieves strong performance-efficiency trade-offs while improving attention expressiveness. Code is available at the following site: \href{https://mike7472727.github.io/laplacianformer.github.io/}{\textcolor{black}{LaplacianFormer }}.

\end{abstract}
\section{Introduction}

The Transformer architecture~\cite{Vaswani2017AttentionIA} has become a fundamental framework for sequence modeling, demonstrating strong performance across a wide range of computer vision tasks~\cite{Jiang2024RBSFormerET,Zhu2020DeformableDD,Yu2024EmbeddingFreeTW,Hou2024ProTransformerRT,Su2024ScanFormerRE}. While its self-attention mechanism effectively captures rich contextual dependencies, its quadratic time and space complexity with respect to sequence length significantly limits scalability to long input sequences~\cite{Keles2022OnTC,Hassani2024FasterNA}.

To address this, a number of linear attention variants have been proposed to approximate softmax attention using kernel-based formulations, thereby reducing complexity to linear~\cite{Katharopoulos2020TransformersAR,Lu2021SOFTST,Chen2021SkyformerRS,Bui2025RevisitingKA,Kashiwagi2021GaussianKS}. Notably, despite differences in implementation, the vast majority of these methods converge on a similar design choice: they rely on Gaussian-like kernels to define attention similarity. This widespread adoption appears to be more of a default convention than a theoretically grounded decision. Indeed, there is a lack of empirical or analytical justification for why the Gaussian kernel is inherently suitable for modeling query-key interactions in attention mechanisms.
\begin{figure}[ht]
\centering
\begin{subfigure}[b]{0.32\linewidth}
    \includegraphics[width=\linewidth]{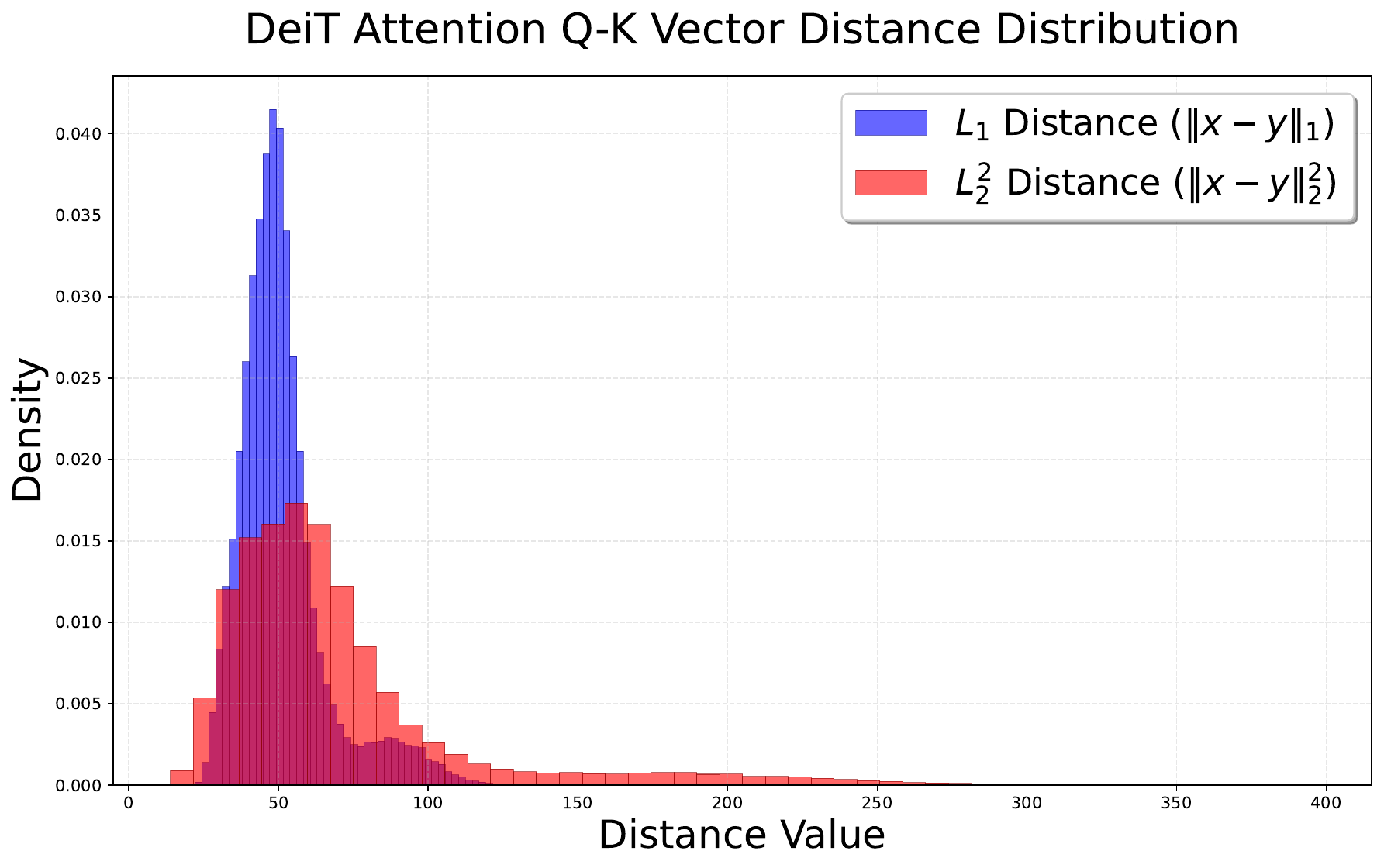}
    \caption{DeiT}
\end{subfigure}
\begin{subfigure}[b]{0.32\linewidth}
    \includegraphics[width=\linewidth]{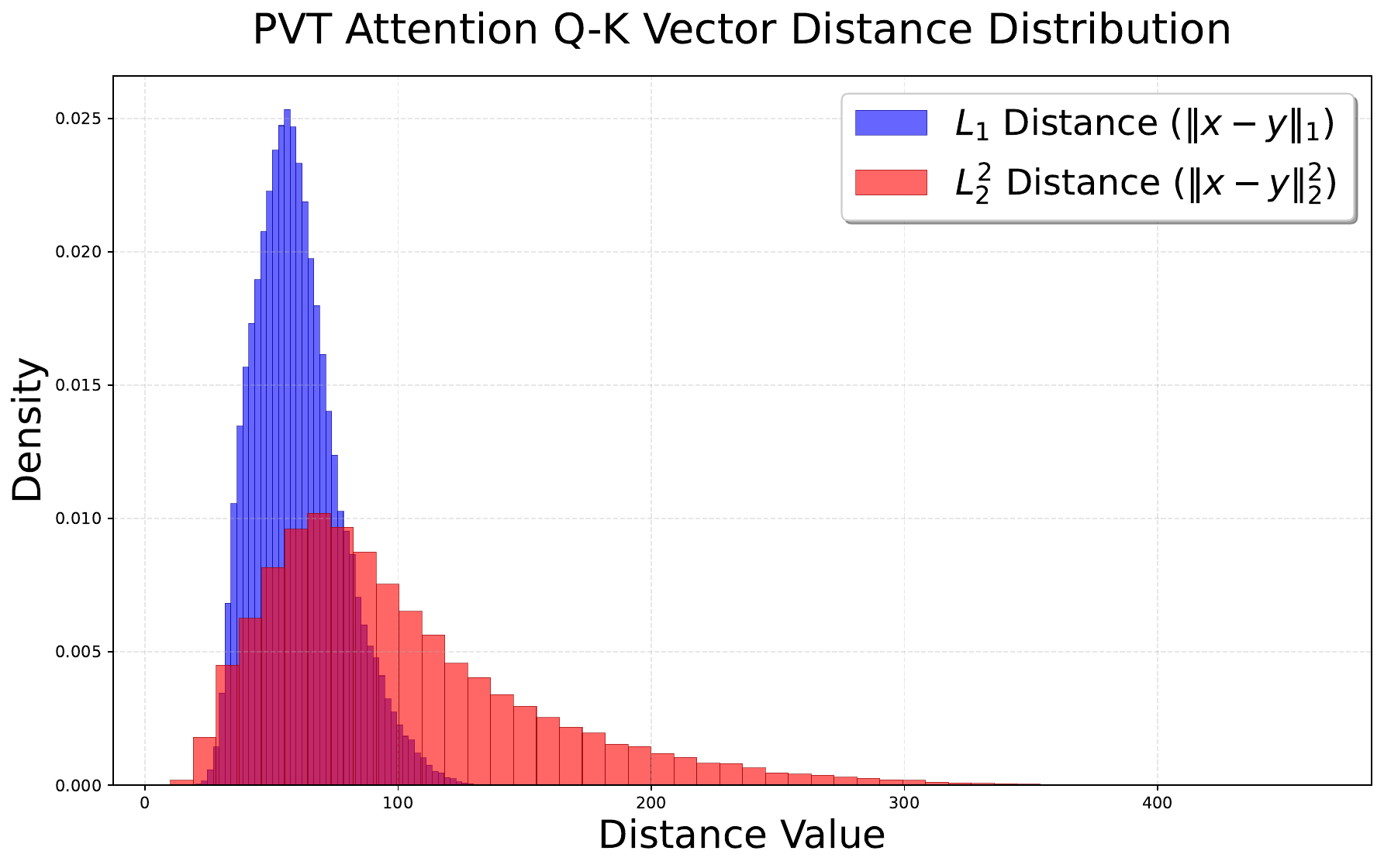}
    \caption{PVT}
\end{subfigure}
\begin{subfigure}[b]{0.32\linewidth}
    \includegraphics[width=\linewidth]{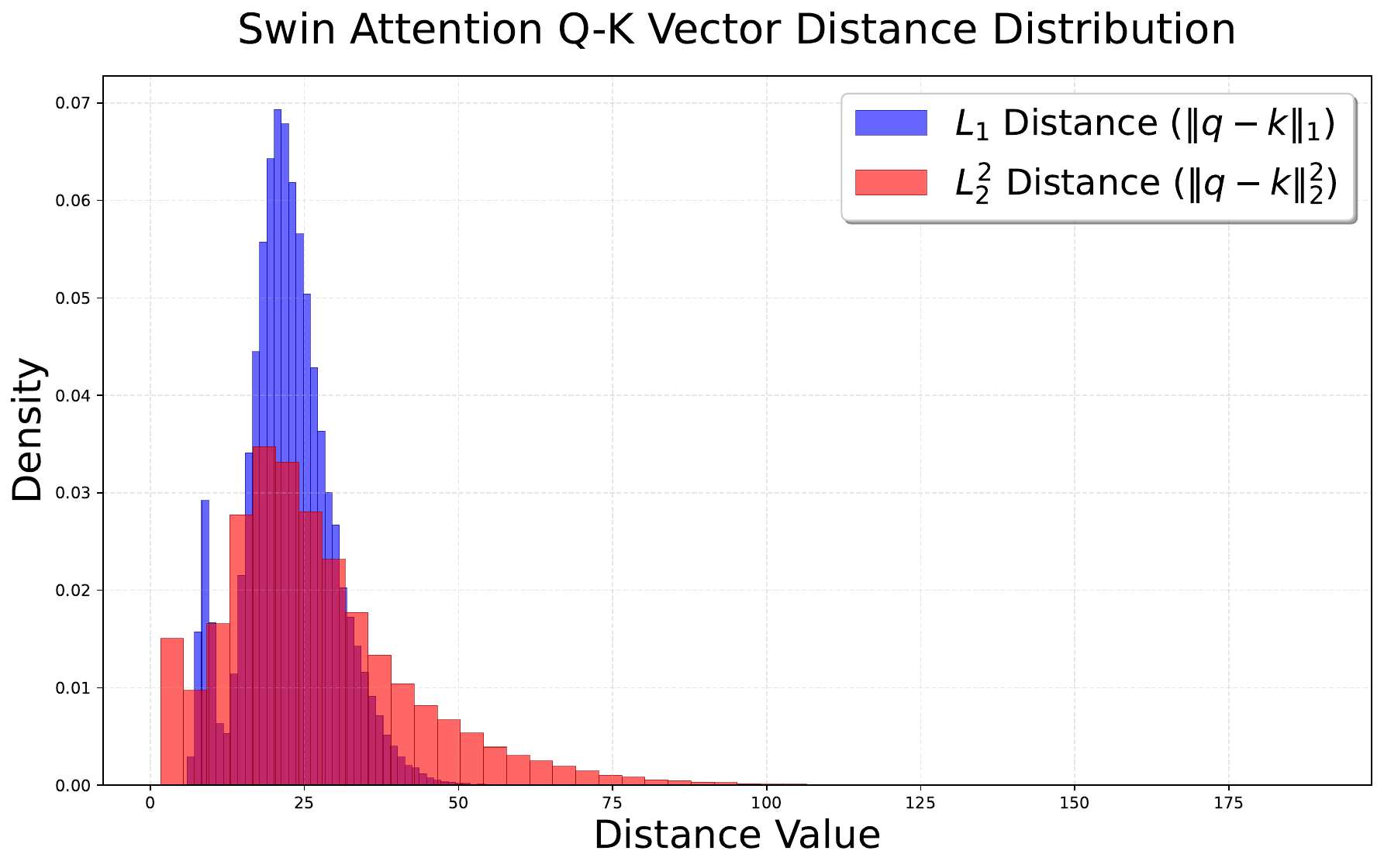}
    \caption{Swin}
\end{subfigure}
\caption{Distributions of $\ell_1$ and $\ell_2^2$ Q-K distances in DeiT, PVT, and Swin Transformers.}
\label{fig:qk_distributions}
\end{figure}

Theoretically, the Gaussian kernel presumes that query-key similarity should decay rapidly with increasing $\ell_2^2$ distance. However, this assumption may not reflect the actual distribution of query-key interactions in vision Transformers. To investigate this issue, we analyze the empirical distribution of query-key distances in DeiT~\cite{Touvron2020TrainingDI}, PVT~\cite{Wang2021PyramidVT}, and Swin~\cite{Liu2021SwinTH}, using official checkpoints on the ImageNet~\cite{Deng2009ImageNetAL} validation set. As shown in Figure~\ref{fig:qk_distributions}, the $\ell_2^2$ distances exhibit a heavy-tailed distribution with high variance and frequent outliers. When passed through the exponential function in the Gaussian kernel, these long-tailed distances will lead to an amplification of the tail effect: outliers dominate the attention map, while moderately relevant keys are overly suppressed. This behavior not only reduces the expressiveness of attention weights but also causes vanishing gradients and unstable optimization, especially during the early stages of training~\cite{Zhang2021DeepLL}. In contrast, $\ell_1$ distances tend to be more concentrated and less sensitive to outliers, providing a more faithful measure of token relevance. This observation motivates the use of the Laplacian kernel, defined as  
$
k(x, y) = \exp\left(-\frac{\|x - y\|_1}{\lambda}\right),
$ 
where $x, y \in \mathbb{R}^d$ are input feature vectors and $\lambda > 0$ is a decay parameter. Compared to the Gaussian kernel, which is based on squared $\ell_2$ distances, the Laplacian kernel exhibits a slower decay rate.
\begin{figure}[htbp] 
    \centering
    \includegraphics[width=0.75\linewidth]{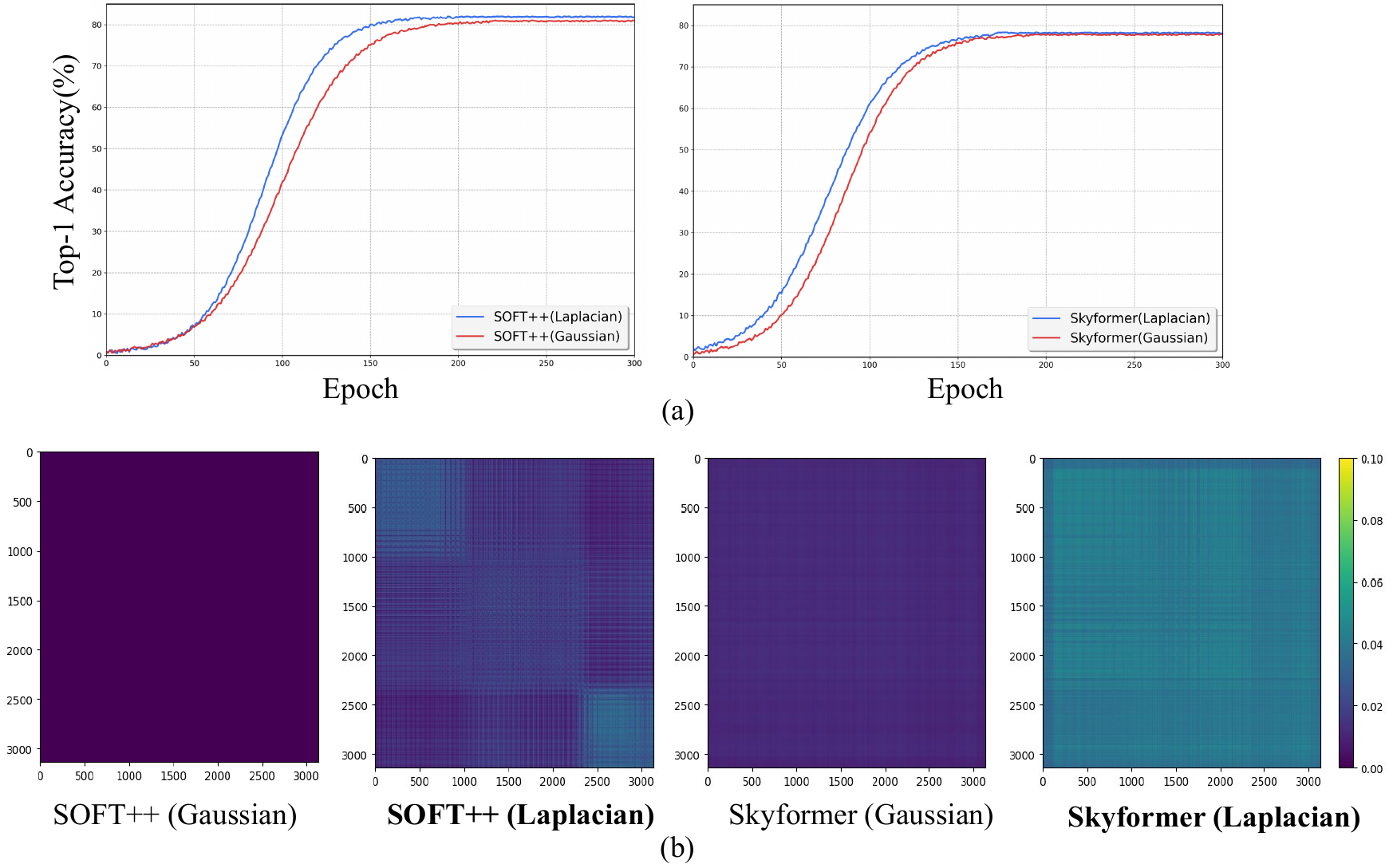}
    \caption{
    (a) Top-1 accuracy (\%) over training epochs on ImageNet. The left plot shows results for SOFT++, and the right plot for Skyformer, each using either a Gaussian or Laplacian kernel for attention computation. Models with Laplacian kernels (blue) converge faster and achieve slightly higher final accuracy compared to their Gaussian counterparts (red). 
    (b) Visual comparison of attention maps with different kernel choices. 
Each pair shows attention maps from the first Transformer block of SOFT++ and Skyformer, where we only replace the Gaussian kernel with a Laplacian kernel, keeping all other components unchanged. Attention matrices computed with Laplacian kernels exhibit more structured patterns and better-conditioned rank profiles.
  }
    \label{fig:example}
\end{figure}

Beyond empirical distributions, we further analyze the gradient behavior of these kernels, which plays a critical role in optimization stability. For the Laplacian kernel, the partial derivative with respect to coordinate $x_i$ is  
$
\frac{\partial k}{\partial x_i} = \frac{1}{\lambda} \cdot \mathrm{sign}(x_i - y_i) \cdot \exp\left(-\frac{\|x - y\|_1}{\lambda}\right),
$ 
while for the Gaussian kernel, it is  
$
\frac{\partial k}{\partial x_i} = \frac{1}{\sigma^2} (x_i - y_i) \cdot \exp\left(-\frac{\|x - y\|_2^2}{2\sigma^2}\right),
$ 
where $\sigma$ denotes the kernel bandwidth.
Notably, the Laplacian kernel maintains non-vanishing gradients even when $x$ and $y$ are nearly identical, owing to the piecewise linear nature of the $\ell_1$ norm. In contrast, the Gaussian kernel’s gradients diminish linearly as $\|x - y\|_2 \to 0$, resulting in vanishing updates that may hinder convergence.  
To empirically validate this theoretical claim, we perform a simple ablation by replacing the Gaussian kernel in two representative models—SOFT++~\cite{Lu2024} and Skyformer~\cite{Chen2021SkyformerRS}—with a Laplacian kernel, keeping all other architectural components unchanged.  
As shown in Figure~\ref{fig:example}, this modification alone leads to significantly faster convergence in both models, supporting the hypothesis that the Laplacian kernel facilitates more stable and efficient learning dynamics.  
Beyond training behavior, we also compare attention maps produced by the two kernels. The same figure also visualizes attention from the first Transformer block of both models. In SOFT++, the Gaussian kernel yields overly sparse attention, while the Laplacian variant produces more expressive and coherent patterns. A similar trend is also observed in Skyformer.  
These results suggest that, beyond its theoretically favorable decay profile, the Laplacian kernel also improves the practical expressiveness of attention maps, particularly in early to mid-stage layers.

Motivated by these findings, we introduce \textbf{LaplacianFormer}, a scalable linear attention framework that replaces the Gaussian-based attention mapping with a Laplacian formulation. To support practical deployment, we develop a CUDA-accelerated implementation that features efficient Laplacian similarity computation and a Newton--Schulz-based inverse solver for fast inference.
\begin{figure}[htbp]
    \centering
    \begin{subfigure}[b]{0.48\columnwidth} 
        \centering
        \includegraphics[width=\textwidth]{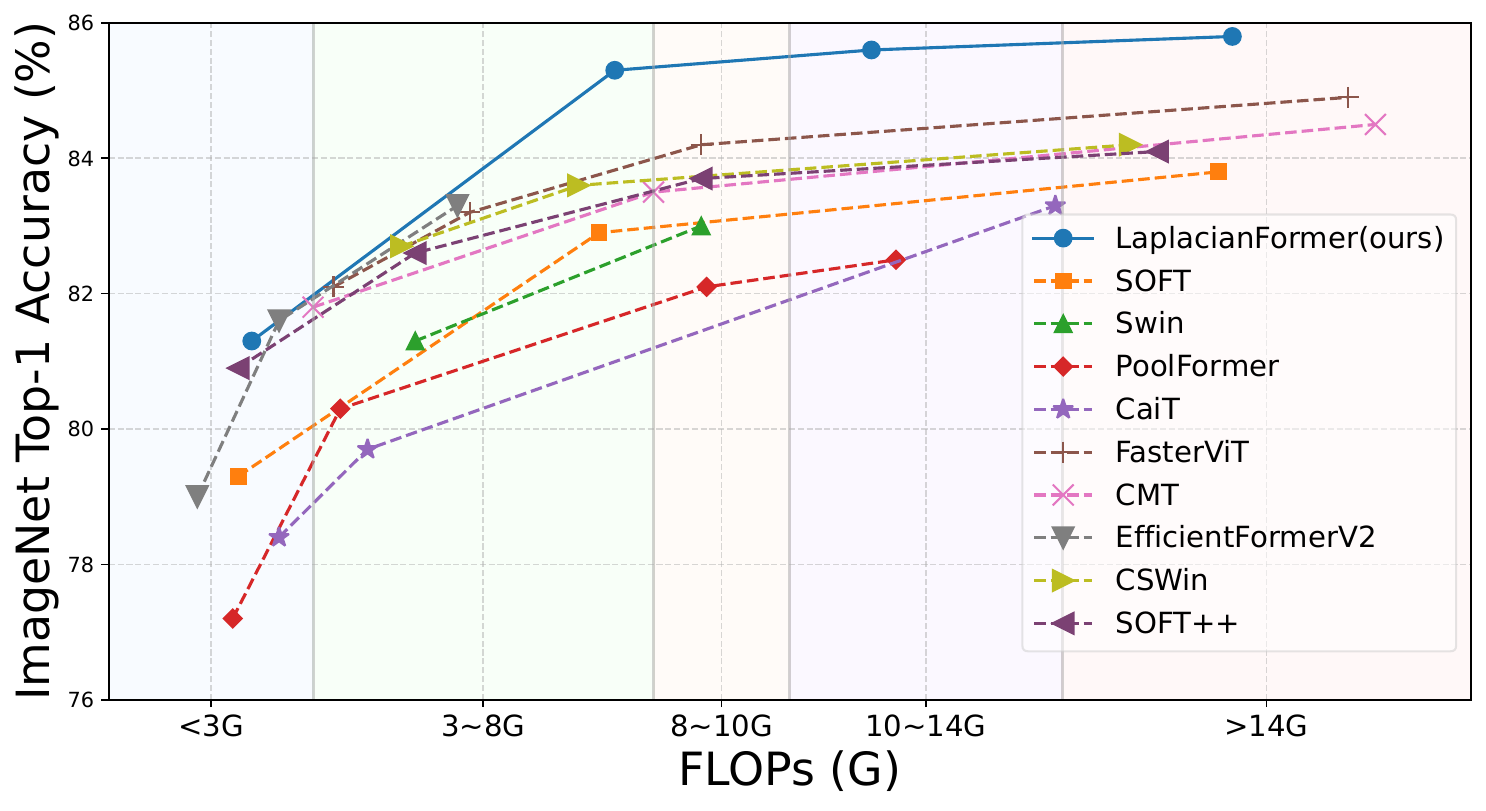}
        \caption{}
        \label{fig:imagenet_performance}
    \end{subfigure}
    \begin{subfigure}[b]{0.48\columnwidth}
        \centering
        \includegraphics[width=\textwidth]{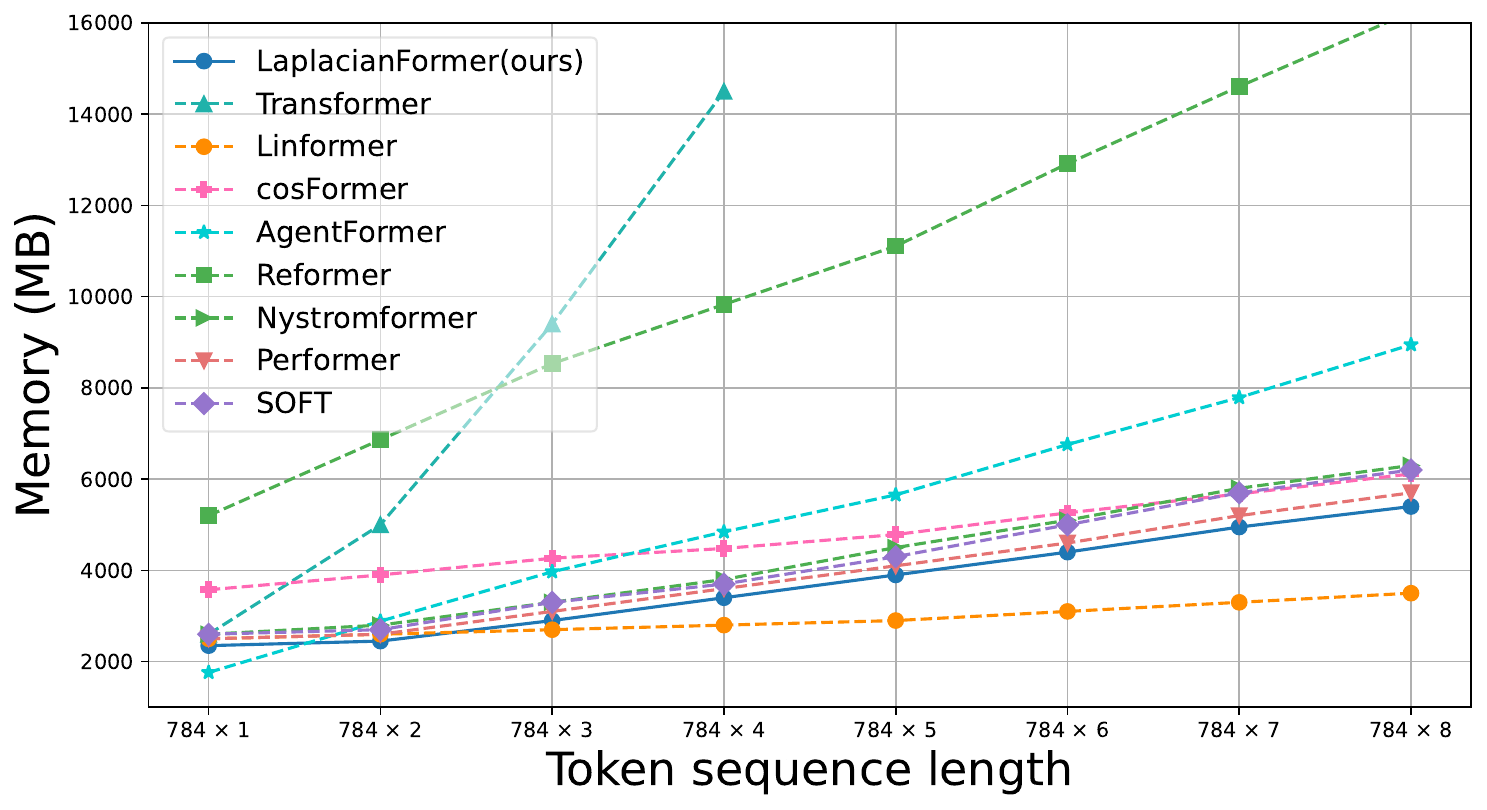}
        \caption{}
        \label{fig:transformer_memory_comparison}
    \end{subfigure}
    \caption{Accuracy and Memory Comparison.(a) Top-1 accuracy vs. FLOPs on ImageNet-1k~\cite{Deng2009ImageNetAL}. LaplacianFormer offers a strong accuracy-efficiency trade-off, outperforming prior models across all FLOPs ranges. (b) GPU memory usage across input lengths. LaplacianFormer shows linear scaling, matches efficient Transformers like Performer and SOFT, and far better than the vanilla Transformer.}
    \label{fig:sota}
\end{figure}
As shown in Figure~\ref{fig:sota}, LaplacianFormer achieves strong performance across accuracy, memory efficiency, and scalability metrics on standard vision benchmarks.

Our main contributions are summarized as follows:
\begin{itemize}
  \item We propose \textbf{LaplacianFormer}, a linear attention model grounded in the Laplacian kernel, which enhances long-range dependency modeling while maintaining scalability and efficiency.
  \item We develop a CUDA-optimized implementation that integrates Laplacian attention with a Newton--Schulz inverse module, significantly improving runtime and memory efficiency.
  \item We validate LaplacianFormer on ImageNet-1k~\cite{Deng2009ImageNetAL} and downstream vision tasks such as object detection and instance segmentation, demonstrating competitive performance across multiple benchmarks.
\end{itemize}

\section{Related work}
\textbf{Vision Transformer with Softmax Attention.}
The Vision Transformer (ViT)~\cite{Dosovitskiy2020AnII} has demonstrated exceptional performance and has been widely adopted for a range of computer vision tasks, including image classification~\cite{Touvron2021GoingDW, Liu2021SwinTH, Liu2021SwinTV, Touvron2022DeiTIR}, object detection~\cite{Zhu2020DeformableDD, Zhang2022DINODW}, and semantic segmentation~\cite{Zheng2020RethinkingSS, Xie2021SegFormerSA, Cheng2021PerPixelCI}. By substituting traditional convolutional operations with self-attention mechanisms, ViT enables the modeling of global dependencies within images, offering a powerful alternative to convolutional neural networks (CNNs). However, a major bottleneck lies in the quadratic time and memory complexity, $\mathcal{O}(n^2)$, of standard softmax attention, which significantly restricts its scalability—especially for high-resolution inputs—and limits its deployment on resource-constrained edge devices.

\textbf{Linear Attention: A Scalable Alternative.} To mitigate the computational overhead of softmax attention, linear attention has emerged as an efficient alternative. While softmax attention requires $\mathcal{O}(N^2 d)$ operations to compute pairwise similarities, linear attention reduces this to $\mathcal{O}(N d^2)$ by replacing the softmax with kernel-based approximations and reordering computations. Specifically, computing $K^\top V$ first decouples the attention process and enables linear scalability. This efficiency gain becomes particularly significant in modern Transformers, where the token count $N$ typically exceeds the channel dimension $d$. Linear attention thus maintains the ability to model long-range dependencies while substantially improving computational efficiency.

Building on this foundation, a number of linear attention variants have been developed to reduce computational cost while enhancing model capacity. Nyströmformer~\cite{Xiong2021NystrmformerAN} approximates softmax attention via Nyström matrix decomposition. SOFT~\cite{Lu2021SOFTST} replaces softmax with a learnable kernel based on low-rank approximations. Skyformer~\cite{Chen2021SkyformerRS} incorporates Gaussian kernels and Nyström sampling to improve scalability in vision tasks, while Gaussian Kernelized Attention~\cite{Kashiwagi2021GaussianKS} applies a similar design to speech decoding. Performer~\cite{choromanski2021rethinking} employs orthogonal random features (FAVOR+) to achieve linear-time softmax approximation. Cosformer~\cite{zhen2022cosformer} replaces softmax with a cosine-based reweighting scheme to achieve linear complexity. Hedgehog~\cite{Zhang2024TheH} introduces structured linear transformations to approximate softmax behavior, providing a unified and scalable alternative. HiViT~\cite{zhanghivit} streamlines hierarchical Transformers by reducing token mixing and applying uniform downsampling.

While differing in architecture, many of these methods share a common reliance on Gaussian kernels to approximate attention weights. In this work, we replace the Gaussian kernel with a Laplace kernel that ensures injectivity and enhances expressiveness, grounded in rigorous theoretical analysis.

\section{Preliminaries}

\subsection{Softmax Self-Attention}

Softmax self-attention is a core operation in transformer models. Given an input sequence $\mathbf{X} \in \mathbb{R}^{N \times d_e}$ of $N$ tokens embedded in a $d_e$-dimensional space, we compute queries, keys, and values via linear projections:$
\mathbf{Q} = \mathbf{X} \mathbf{W}_Q, \mathbf{K} = \mathbf{X} \mathbf{W}_K, \mathbf{V} = \mathbf{X} \mathbf{W}_V$,
where $\mathbf{W}_Q, \mathbf{W}_K, \mathbf{W}_V \in \mathbb{R}^{d_e \times d}$ are learnable parameters and $\mathbf{Q}, \mathbf{K}, \mathbf{V} \in \mathbb{R}^{N \times d}$.
The standard scaled dot-product attention for token $i$ is:
\begin{equation}
\mathrm{Attention}(\mathbf{Q}, \mathbf{K}, \mathbf{V})_i =
\frac{ 
  \sum_{j=1}^N \exp\left(\frac{\mathbf{q}_i^\top \mathbf{k}_j}{\sqrt{d}}\right)\, \mathbf{v}_j
}{
  \sum_{j=1}^N \exp\left(\frac{\mathbf{q}_i^\top \mathbf{k}_j}{\sqrt{d}}\right)
}
\label{eq:attn-elementwise}
\end{equation}

and in matrix form:
\begin{equation}
\mathrm{Attention}(\mathbf{Q}, \mathbf{K}, \mathbf{V}) =
\mathrm{Softmax}\left( \frac{\mathbf{Q} \mathbf{K}^\top}{\sqrt{d}} \right) \mathbf{V}
\label{eq:attn-matrix}
\end{equation}

This formulation computes a similarity matrix $\mathbf{QK}^\top \in \mathbb{R}^{N \times N}$, resulting in $\mathcal{O}(N^2 d)$ complexity due to all pairwise interactions.
To reduce cost, consider removing the softmax. Without it, attention simplifies to $\left( \frac{\mathbf{Q} \mathbf{K}^\top}{\sqrt{d}} \right) \mathbf{V},$
which can be reordered as $\mathbf{Q} (\mathbf{K}^\top \mathbf{V})$ using associativity. This avoids forming the large $N \times N$ matrix and reduces complexity to $\mathcal{O}(N d^2)$, linear in $N$ if $d$ is small. This insight underlies linear attention, which replaces softmax with associative operations for improved efficiency. Figure~\ref{fig:softmax-vs-linear} compares softmax and linear self-attention.

\begin{figure}[!htbp]
    \centering
    \includegraphics[width=\linewidth]{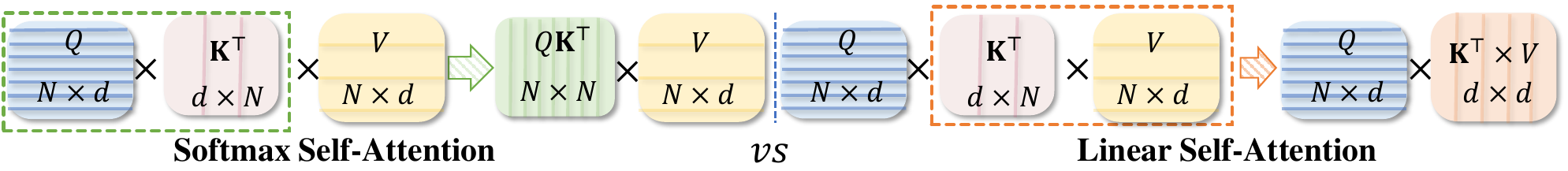}
    \caption{Comparison between Softmax Self-Attention (left) and Linear Self-Attention (right). 
    The former computes a full $N \times N$ similarity matrix, while the latter enables associativity through kernel decomposition, reducing the complexity from $\mathcal{O}(N^2)$ to $\mathcal{O}(N)$.}
    \label{fig:softmax-vs-linear}
\end{figure}

\subsection{Linear Self-Attention}
\label{sec:Linear}

Linear self-attention reformulates the attention mechanism by approximating the similarity computation through kernel-based feature mappings. Specifically, let \( \phi(\cdot) \) denote a kernel function, and define the similarity between a query \( \mathbf{q}_i \) and a key \( \mathbf{k}_j \) as: $\mathrm{Sim}(\mathbf{q}_i, \mathbf{k}_j) = \phi(\mathbf{q}_i)\, \phi(\mathbf{k}_j)^\top.$

The kernel-based formulation replaces the exponential dot product with a more general similarity function, enabling efficient reordering of computations and eliminating the softmax operation. The attention output for the \( i \)-th query can be written as:
\begin{equation}
\mathrm{Attention}(\mathbf{Q}, \mathbf{K}, \mathbf{V})_i
= 
\frac{
  \phi\bigl(\mathbf{q}_i\bigr) 
  \Bigl(\sum_{j=1}^N \phi\bigl(\mathbf{k}_j\bigr)^\top \mathbf{v}_j \Bigr)
}{
  \phi\bigl(\mathbf{q}_i\bigr) 
  \Bigl(\sum_{j=1}^N \phi\bigl(\mathbf{k}_j\bigr)^\top \Bigr)
}.
\end{equation}
Since the key-value summaries 
\( \sum_{j=1}^N \phi(\mathbf{k}_j)^\top \mathbf{v}_j \) and 
\( \sum_{j=1}^N \phi(\mathbf{k}_j)^\top \) 
are independent of the query, they can be precomputed, allowing each attention output to be computed in linear time.

\section{Method}
\subsection{LaplacianFormer}

Our LaplacianFormer instantiates the general kernel attention framework described in Section~\ref{sec:Linear} using a novel Laplace-based transformation inspired by recent work on attention injectivity~\cite{han2024inline}. Instead of directly using the Laplacian kernel as a similarity score, we construct a normalized kernel representation for each query \( \mathbf{q}_i \) to enhance feature discrimination:

\begin{equation}
\mathbf{z}_i = \boldsymbol{\Sigma}^{-\frac{1}{2}} \left( 
  \left[ k(\mathbf{q}_i, \mathbf{k}_1), \dots, k(\mathbf{q}_i, \mathbf{k}_N) \right]^\top
  - \frac{1}{N} \sum_{j=1}^N k(\mathbf{q}_i, \mathbf{k}_j)
\right) + \frac{1}{N},
\label{eq:injective-map}
\end{equation}

where \( k(\mathbf{q}, \mathbf{k}) = \exp\left(-\frac{\|\mathbf{q} - \mathbf{k}\|_1}{\lambda}\right) \) denotes the Laplacian kernel. The whitening matrix \( \boldsymbol{\Sigma}^{-1/2} \in \mathbb{R}^{N \times N} \) is ideally computed from the covariance of query–key similarity vectors \( \mathbf{g}_i \in \mathbb{R}^N \), where each $
\mathbf{g}_i = [k(\mathbf{q}_i, \mathbf{k}_1), \dots, k(\mathbf{q}_i, \mathbf{k}_N)]^\top.
$

In practice, computing the full inverse square root \( \boldsymbol{\Sigma}^{-1/2} \) is computationally prohibitive for long sequences, requiring eigendecomposition with \( \mathcal{O}(N^3) \) time and \( \mathcal{O}(N^2) \) memory. To mitigate this, we approximate the whitening operation with a diagonal estimator that normalizes each feature dimension independently across a batch of query–key vectors \( \{\mathbf{g}_i\}_{i=1}^B \), where \( B \) is the batch size.

For each dimension \( j \in \{1, \dots, N\} \), we compute the empirical mean and variance:
\begin{equation}
\mu_j = \frac{1}{B} \sum_{i=1}^{B} \mathbf{g}_{ij}, \quad
\sigma_j^2 = \frac{1}{B} \sum_{i=1}^{B} \left( \mathbf{g}_{ij} - \mu_j \right)^2.
\end{equation}

We then normalize each element of the similarity vector:
$\tilde{\mathbf{g}}_{ij} = \frac{\mathbf{g}_{ij} - \mu_j}{\sqrt{\sigma_j^2 + \varepsilon}}$,
where \( \varepsilon \) is a small constant added for numerical stability. This corresponds to a diagonal whitening matrix:
\begin{equation}
\mathbf{D}^{-1/2} = \mathrm{diag}\left( \frac{1}{\sqrt{\sigma_1^2 + \varepsilon}}, \dots, \frac{1}{\sqrt{\sigma_N^2 + \varepsilon}} \right).
\end{equation}

This approximation preserves the centering and scaling effects of full whitening, improves stability, and reduces both time and memory complexity from quadratic to linear in \( N \), making it compatible with efficient kernelized attention. For completeness, we define the kernel similarity matrix among keys as $\mathbf{G}_{ij} \&= k(\mathbf{k}_i, \mathbf{k}_j),$ $\mathbf{\Sigma}_{\text{key}} \&= \mathbf{P} \mathbf{G} \mathbf{P}^\top,$ $\text{with } \mathbf{P} = \mathbf{I} - \frac{1}{N} \mathbf{1} \mathbf{1}^\top.$ Although not directly used in normalization, the key–key covariance \( \mathbf{\Sigma}_{\text{key}} \) provides a useful interpretation of the kernel structure.

We prove in the appendix that the transformation in Eq.~\eqref{eq:injective-map} is injective under mild assumptions, ensuring that distinct queries yield distinct outputs. This injectivity property aligns with the behavior of softmax attention, which is inherently injective and yields full-rank attention maps that preserve fine-grained token distinctions~\cite{han2024inline}. The final attention output incorporates both global interactions via kernelized similarity and local context modeling through depth-wise convolution. Specifically, we compute
\begin{equation}
\mathrm{Attention}(\mathbf{Q}, \mathbf{K}, \mathbf{V}) = \mathbf{Z} \mathbf{V} + \mathrm{DWC}(\mathbf{V}),
\end{equation}

where \( \mathbf{Z} \in \mathbb{R}^{N \times N} \) stacks each \( \mathbf{z}_i^\top \) as a row, and \( \mathrm{DWC}(\cdot) \) denotes a depth-wise convolution applied over the value sequence \( \mathbf{V} \).

\subsection{Nyström Approximation for Laplacian Kernel}

To efficiently compute Laplacian kernel-based attention, we adopt a Nyström approximation~\cite{Williams2000UsingTN,Xiong2021NystrmformerAN}. The Nyström method approximates the kernel matrix \( \mathbf{G} \) by selecting a small set of landmark keys and computing a rank-reduced estimate \( \widetilde{\mathbf{G}} \in \mathbb{R}^{N \times N} \), defined as $\widetilde{\mathbf{G}} = \mathbf{C} \mathbf{W}^{\dagger} \mathbf{C}^\top$,
where \( \mathbf{C} \in \mathbb{R}^{N \times m} \) is the matrix of Laplacian kernel similarities between all queries and a selected subset of \( m \ll N \) landmark keys, \( \mathbf{W} \in \mathbb{R}^{m \times m} \) contains pairwise Laplacian kernel similarities among the \( m \) selected landmark keys, and \( \mathbf{W}^{\dagger} \) denotes the Moore–Penrose pseudoinverse of \( \mathbf{W} \). More specifically, the \( (i,\ell) \)-th entry of \( \mathbf{C} \) is defined as:
\begin{equation}
\mathbf{C}_{i\ell} = k(\mathbf{q}_i, \tilde{\mathbf{k}}_\ell)
= \exp\left(-\frac{1}{\lambda} \left\| \mathbf{q}_i - \tilde{\mathbf{k}}_\ell \right\|_1 \right),
\label{eq:C}
\end{equation}

where \( \mathbf{q}_i \) is the query vector of the \( i \)-th token and \( \tilde{\mathbf{k}}_\ell \in \{\mathbf{k}_1, \dots, \mathbf{k}_N\} \) is the \( \ell \)-th landmark key, while the \( (\ell, \ell') \)-th entry of \( \mathbf{W} \) is computed as:
\begin{equation}
\mathbf{W}_{\ell \ell'} = k(\tilde{\mathbf{k}}_\ell, \tilde{\mathbf{q}}_{\ell'})
= \exp\left(-\frac{1}{\lambda} \left\| \tilde{\mathbf{k}}_\ell - \tilde{\mathbf{q}}_{\ell'} \right\|_1 \right),
\label{eq:W}
\end{equation}
where \( \tilde{\mathbf{k}}_\ell \) and \( \tilde{\mathbf{q}}_{\ell'} \) are the landmark key and query vectors selected by Nyström sampling, respectively.

The process for computing the low-rank Laplacian kernel via Nyström approximation is outlined in Algorithm~\ref{alg:laplacianformer}. In Line 2, the sampling function $f_s$ selects $m \ll N$ landmark tokens from the full set of queries and keys, forming the landmark matrices $\widetilde{\mathbf{Q}}, \widetilde{\mathbf{K}} \in \mathbb{R}^{m \times d}$. Lines 3–5 perform the core kernel operations: Line 3 computes the landmark kernel matrix $\mathbf{W}$ (Eq.~\ref{eq:W}), Line 4 computes the query-to-landmark kernel matrix $\mathbf{C}$ (Eq.~\ref{eq:C}), and Line 5 applies the Nyström approximation using $\mathbf{W}^\dagger$ to obtain the final attention matrix $\hat{\mathbf{S}}$.

\begin{algorithm}
\normalsize
\caption{Laplacian Kernel with Nyström Approximation}
\label{alg:laplacianformer}
\begin{algorithmic}[1]
\State \textbf{Input:} Queries $Q \in \mathbb{R}^{N \times d}$, Keys $K \in \mathbb{R}^{N \times d}$, Nyström sampling function $f_s$
\State \textbf{Sampling:} $\tilde{Q}, \tilde{K} \gets f_s(Q), f_s(K)$ \Comment{Select $m \ll n$ landmark points}
\State $W \gets \exp\left(-\frac{1}{\lambda} \|\tilde{Q} \ominus \tilde{K}\|_1\right)$ \Comment{Kernel matrix on sampled points}
\State $C \gets \exp\left(-\frac{1}{\lambda} \|Q \ominus \tilde{K}\|_1\right)$ \Comment{Cross-kernel between all queries and landmarks}
\State $\hat{G} \gets C W^{\dagger} C^{\top}$ \Comment{Low-rank approximation of full kernel matrix}
\State \textbf{Output:} $\hat{G}$
\end{algorithmic}
\end{algorithm}

\begin{algorithm}
\normalsize 
\caption{\normalsize Newton--Schulz Iteration for Approximating \( \mathbf{W}^{\dagger} \)}
\label{alg:newton-schulz-w-pinv}
\begin{algorithmic}[1]
\setlength{\parsep}{0pt}   
\setlength{\partopsep}{0pt} 
\State \textbf{Input:} Landmark kernel matrix \( \mathbf{W} \in \mathbb{R}^{m \times m} \), number of iterations \( \mathcal{T} \in \mathbb{Z}^+ \)
\State Add small perturbation: \( \mathbf{W} \gets \mathbf{W} + \epsilon \mathbf{I} \), where \( \epsilon > 0 \)
\State Initialize scaling factor: \( \boldsymbol{\alpha} \gets \frac{2}{\|\mathbf{W}\|_2} \)
\State Initialize: \( \mathbf{X}_0 \gets \boldsymbol{\alpha} \mathbf{W}^\top \)
\For{ \( k = 1 \) to \( \mathcal{T} \) }
    \State \( \mathbf{X}_k \gets \mathbf{X}_{k-1} (2\mathbf{I} - \mathbf{W} \mathbf{X}_{k-1}) \)
\EndFor
\State \textbf{Output:} Approximate pseudoinverse \( \mathbf{X}_\mathcal{T} \approx \mathbf{W}^{\dagger} \)
\end{algorithmic}
\end{algorithm}

\textbf{Laplacian Kernel Inversion via Newton--Schulz Iteration.} 
To efficiently and stably approximate the inverse of the landmark kernel matrix 
$\mathbf{W} \in \mathbb{R}^{m \times m}$, which is symmetric and positive semi-definite, 
we use the Newton--Schulz iteration. Since convergence requires $\mathbf{W}$ to be strictly 
positive definite, we apply a small diagonal perturbation 
$\mathbf{W} \leftarrow \mathbf{W} + \epsilon \mathbf{I}$, with $\epsilon > 0$, preserving the 
structure while ensuring stability. Unlike inversion or SVD-based methods, Newton--Schulz relies only on matrix multiplications and 
additions, making it GPU-friendly. The iteration starts with 
$\mathbf{X}_0 = \boldsymbol{\alpha} \mathbf{W}^\top$, where 
$\boldsymbol{\alpha} = \frac{2}{\|\mathbf{W}\|_2}$ ensures 
$\|\mathbf{I} - \boldsymbol{\alpha} \mathbf{W} \mathbf{W}^\top\| < 1$. The update rule is: 
$\mathbf{X}_{k+1} = \mathbf{X}_k (2\mathbf{I} - \mathbf{W} \mathbf{X}_k).$ 
The full algorithm is detailed in Algorithm~\ref{alg:newton-schulz-w-pinv}.

\textbf{Sampling Strategies for Landmark Selection.}  
To efficiently approximate attention, we adopt a pooling-based landmark selection strategy inspired by PVTv2~\cite{Wang2021PVTVI}.  
The query tensor $\mathbf{Q} \in \mathbb{R}^{B \times H \times N \times d}$ is reshaped into a spatial map $\mathbb{R}^{B \cdot H \times d \times H' \times W'}$, where $N = H' \times W'$.  
We apply average pooling with kernel size $r$ and stride $r$ to aggregate each $r \times r$ region into a landmark token, yielding $\frac{H'}{r} \times \frac{W'}{r}$ tokens per head.

We also explored a depthwise convolution-based selection strategy, in which each $r \times r$ region is processed by a lightweight filter to extract local structure. While this approach offers greater expressiveness, it yielded no significant improvement over average pooling in our experiments. Given its higher computational cost and additional parameters, we adopt average pooling by default.

\textbf{Convergence Guarantee and Complexity Analysis.} 
The Newton–Schulz iteration is guaranteed to converge for strictly positive definite matrices; this condition is satisfied by applying a small diagonal perturbation to $\mathbf{W}$. 
Our method achieves linear time and space complexity $\mathcal{O}(n)$ with respect to the input length $n$. 
A complete complexity analysis and proof of convergence are provided in the appendix.

\subsection{CUDA Acceleration}

\begin{figure}
    \centering
    \includegraphics[width=1.0\linewidth]{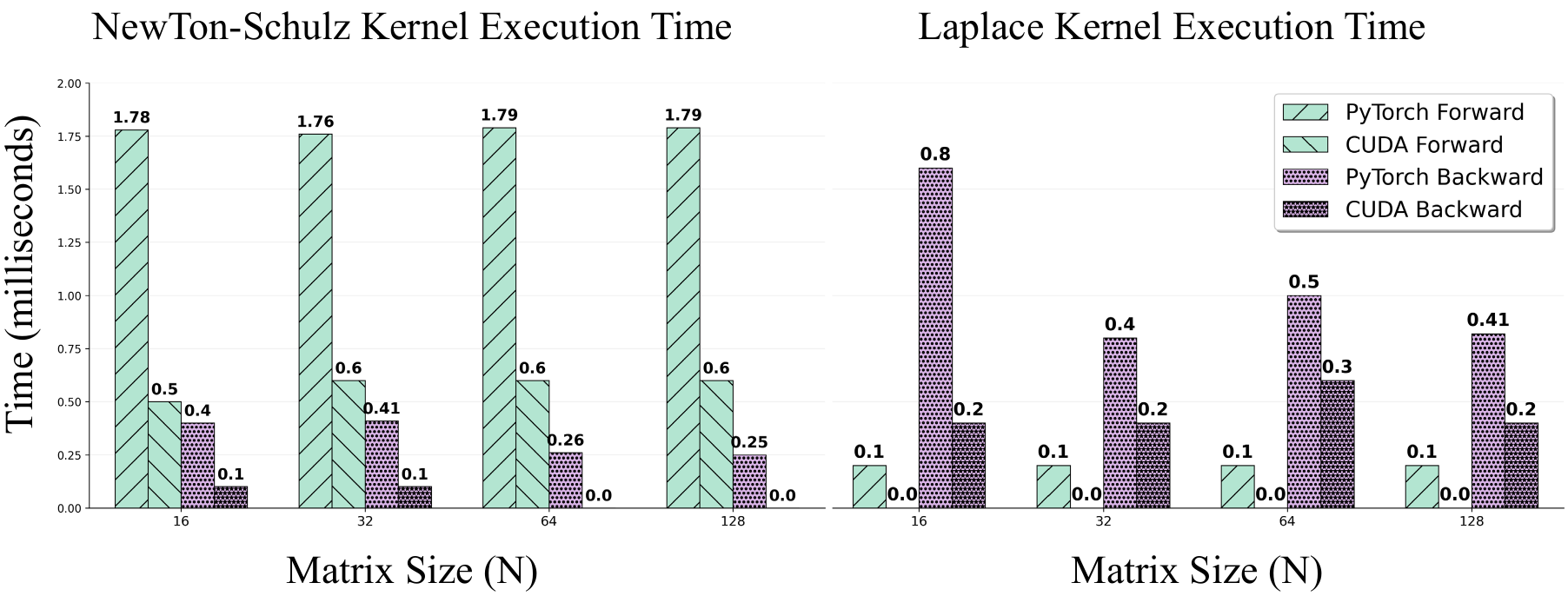}
    \caption{\normalsize \textbf{Execution time breakdown of custom CUDA kernels.}  
    Comparison of forward and backward execution time for Newton--Schulz iteration (left) and Laplacian kernel (right) across different matrix sizes (batch = 1, 2 heads, 32 channels). CUDA execution times ($<0.05$ms) are shown as 0.0 due to timing resolution limits.}
    \label{fig:kernel-execution-breakdown}
\end{figure}

The Laplacian kernel fuses distance computation and exponential transformation into a single operation, reducing global memory access. For Newton--Schulz iteration, we optimize matrix multiplications via tiling and register reuse.

To assess the effectiveness of our CUDA acceleration, we compare the execution time of both Laplacian kernel evaluation and Newton--Schulz iteration against their PyTorch counterparts~\cite{paszke2017automatic}, with and without custom CUDA kernels. As shown in Figure~\ref{fig:kernel-execution-breakdown}, our implementation consistently outperforms the baseline across various matrix sizes. The speedup is particularly prominent in backward passes, which benefit from precomputed gradients and in-place memory reuse. Numerical accuracy comparisons are provided in the appendix.

\section{Experiments}

\begin{table}[t]
\centering
\scriptsize
\caption{Performance comparison with state-of-the-art models on ImageNet.}
\label{tab:comparison}
\begin{tabular}{clcccc}
\toprule
FLOPs  range& Model & Params & FLOPs & Top-1 \%$\uparrow$ & Image Size \\
\hline
\multirow{5}{*}{$<3$G} &  Agent-Deit-T~\cite{Han2023AgentAO} & 6.0M & 1.2G & 74.9 & 224 \\
& VRWKV-T~\cite{Duan2024VisionRWKVEA} & 6.2M & 1.2G & 75.1& 256 \\
& PVT-T-PolaFormer~\cite{meng2025polaformer} & 12M & 2.0G & 78.8 & 224 \\
& FL-PVTv2-B1~\cite{Han2023FLattenTV} & 13M & 2.2G & 79.5 & 224 \\
& BiFormer-T~\cite{zhu2023biformer} & 13.1M & 2.2G & 81.4 & 224 \\
\rowcolor{blue!20} & LaplacianFormer-Tiny & 12.1M & 2.1G & \textbf{81.4} & 224 \\
\hline
\multirow{7}{*}{3$\sim$8G} & InLine-CSwin-S~\cite{han2024inline} & 33M & 6.8G & 83.8 & 224 \\
& SViT-S~\cite{huang2023stvit} & 25M & 4.4G & 83.6 & 224 \\
& BiFormer-S~\cite{zhu2023biformer} & 25.5M & 4.5G & 83.8 & 224 \\
& HiViT-T~\cite{zhanghivit} & 19M & 4.6G & 82.1 & 224 \\
& Agent-PVT-S~\cite{Han2023AgentAO} & 20.6M & 4.0G & 82.2 & 224 \\
\rowcolor{blue!20} & LaplacianFormer-Small & 25.7M & 4.8G & \textbf{83.8} & 224 \\
\hline
\multirow{4}{*}{8$\sim$10G} & SViT-B~\cite{huang2023stvit} & 52M & 9.9G & 84.8& 224 \\
& SOFT++-Medium~\cite{Lu2024} & 45M & 7.2G & 83.7 & 224 \\
&  BiFormer-B~\cite{zhu2023biformer} & 56.8M & 9.8G & 84.3& 224 \\
& Swin-S-PolaFormer~\cite{meng2025polaformer} & 50M & 8.7G & 83.6 & 224 \\
& SLAB-Swin-S~\cite{Guo2024SLABET} & - & 8.7G & 81.8 & 224 \\
\rowcolor{blue!20} & LaplacianFormer-Medium & 46.3M & 7.43G & \textbf{85.3}& 224 \\
\hline
\multirow{5}{*}{10$\sim$14G} & StructViT-B-8-1~\cite{Kim2024LearningCS} & 52M & 12G & 84.3 & 224 \\
& SOFT++-Large~\cite{Lu2024} & 64M & 11G & 84.1 & 224 \\
& NAT-B~\cite{hassani2023neighborhood} & 90M & 13.7G & 84.3& 224 \\
& MogaNet-L~\cite{iclr2024MogaNet} & 82.5M & 15.9G & 84.7& 224 \\
& FLatten-CSwin-S~\cite{Han2023FLattenTV} & 35M & 6.9G & 83.6 & 224 \\

\rowcolor{blue!20} & LaplacianFormer-Large & 63.1M & 11.2G & \textbf{85.6}& 224 \\
\hline
\multirow{4}{*}{$>$14G} & VRWKV-B~\cite{Duan2024VisionRWKVEA} & 93.7M & 18.2G & 82.0& 224 \\
& SViT-L~\cite{huang2023stvit} & 95M & 15.6G & 85.3& 224 \\
& MLLA-B~\cite{han2024demystify} & 96M & 16.2G & 85.3& 224 \\
& HiViT-B~\cite{zhanghivit} & 66M & 15.9G & 83.8& 224 \\
\rowcolor{blue!20} & LaplacianFormer-Huge & 78.5M & 15.5G & \
\textbf{85.8} & 224 \\
\bottomrule
\end{tabular}

\label{tab:comparison}
\end{table}

\subsection{Image Classification}

\paragraph{Datasets and model architectures.}
We evaluate our model on the ImageNet-1K dataset~\cite{Deng2009ImageNetAL}, which contains 1.28M training and 50K validation images across 1000 classes. Built on the PVT architecture~\cite{Wang2021PyramidVT}, our LaplacianFormer re-designs the self-attention mechanism by constructing an injective attention function based on the Laplacian kernel. To ensure training efficiency, we implement two custom CUDA kernels: one for computing the Laplacian kernel matrix and another for performing Newton–Schulz iteration to approximate the inverse. Additionally, RoPE~\cite{SU2024127063} is adopted for positional encoding. All other settings follow the original PVT configuration. Training is performed with a batch size of 1024 on multiple NVIDIA H800 GPUs.

\textbf{Comparison.} We compare the Top-1 accuracy and computational cost of our LaplacianFormer against state-of-the-art Vision Transformers. As shown in Table~\ref{tab:comparison}, models are grouped by FLOPs:$<$1G, 1–3G, 3–5G, 5–10G, and $>$10G. LaplacianFormer consistently achieves the highest Top-1 accuracy across all FLOP ranges. This result demonstrates the superiority of LaplacianFormer over existing methods.

\subsection{Object Detection and Instance Segmentation}
\textbf{Results.}~Table~\ref{tab:backbone_comparison} summarizes the comparison results under the $1\times$ schedule for both Mask R-CNN~\cite{He2017MaskR} and RetinaNet~\cite{Lin2017FocalLF}. Across all scales, LaplacianFormer consistently surpasses previous backbone designs. For instance, LaplacianFormer-Tiny achieves 43.2~$\mathrm{AP}^b$ and 40.3~$\mathrm{AP}^m$ under Mask R-CNN, outperforming SOFT++-Tiny and FL-PVT-T. Under RetinaNet, it further achieves 42.5~$\mathrm{AP}^b$, ranking first among all tiny-scale counterparts. As the model size increases, LaplacianFormer-Medium yields 48.0~$\mathrm{AP}^b$ and 43.5~$\mathrm{AP}^m$, establishing a new state-of-the-art within the medium-sized category. These results highlight the strong generalization and detection capabilities enabled by our Laplacian kernel attention mechanism.

\begin{table}[htbp]
\centering
\scriptsize
\caption{Comparison to other backbones using RetinaNet and Mask R-CNN with ``1×'' schedule.}

\setlength{\tabcolsep}{4pt} 
\definecolor{lightcyan}{RGB}{175,238,238} 

\begin{tabular}{l|cccccc|cccccc}
\toprule
\multirow{2}{*}{Backbone} & \multicolumn{6}{c|}{Mask R-CNN 1×} & \multicolumn{6}{c}{RetinaNet 1×} \\
& $AP^b$ & $AP^b_{50}$ & $AP^b_{75}$ & $AP^m$ & $AP^m_{50}$ & $AP^m_{75}$ & $AP^b$ & $AP^b_{50}$ & $AP^b_{75}$ & $AP^b_S$ & $AP^b_M$ & $AP^b_L$ \\
\midrule
Swin-T-PRepBN~\cite{Guo2024SLABET} & 42.9 & 65.8 & 46.8 & 39.3 & 62.6 & 41.9 & -- & -- & -- & -- & -- & -- \\
FL-PVT-T~\cite{Han2023FLattenTV} & 38.2 & 61.6 & 41.9 & 37.0 & 57.6 & 39.0 & -- & -- & -- & -- & -- & -- \\
SOFT++-Tiny~\cite{Lu2024} & 41.2 & 63.7 & 44.7 & 38.2 & 61.0 & 41.0 & 41.9 & 62.7 & 44.7 & 27.8 & 45.4 & 55.6 \\
\rowcolor{blue!20}
LaplacianFormer-Tiny & \textbf{43.2} & \textbf{66.1} & \textbf{47.2} & \textbf{40.3} & \textbf{63.0} & \textbf{42.9} & \textbf{42.5} & \textbf{64.1} & \textbf{46.4} & \textbf{29.1} & \textbf{46.9} & \textbf{57.8} \\
\midrule
PVT-S-PolaFormer~\cite{meng2025polaformer} & 43.9 & 66.1 & 47.9 & 40.2 & 63.1 & 43.0 & 43.2 & 64.1 & 46.4 & -- & -- & -- \\
InLine-PVT-S~\cite{han2024inline} & 43.4 & 66.4 & 47.1 & 40.1 & 63.1 & 43.3 & -- & -- & -- & -- & -- & -- \\
SOFT++-Small~\cite{Lu2024} & 43.8 & 66.0 & 47.5 & 40.1 & 63.0 & 43.0 & 43.7 & 64.9 & 46.8 & 28.7 & 47.4 & 57.6 \\
\rowcolor{blue!20}
LaplacianFormer-Small & \textbf{45.8} & \textbf{68.2} & \textbf{49.8} & \textbf{42.0} & \textbf{65.1} & \textbf{45.2} & \textbf{45.5} & \textbf{66.8} & \textbf{49.1} & \textbf{30.7} & \textbf{51.8} & \textbf{59.5} \\
\midrule
Agent-PVT-M~\cite{Han2023AgentAO} & 45.9 & 67.8 & 50.4 & 42.0 & 65.0 & 45.4 & -- & -- & -- & -- & -- & -- \\
FL-Swin-M~\cite{Han2023FLattenTV} & 44.0 & 66.4 & 48.0 & 40.3 & 63.4 & 43.5 & -- & -- & -- & -- & -- & -- \\
SOFT++-Medium~\cite{Lu2024} & 46.6 & 67.8 & 51.2 & 42.0 & 64.8 & 45.2 & 44.3 & 64.7 & 47.4 & 29.0 & 48.2 & 59.9 \\
\rowcolor{blue!20}
LaplacianFormer-Medium & \textbf{48.0} & \textbf{70.3} & \textbf{52.5} & \textbf{43.5} & \textbf{65.8} & \textbf{46.5} & \textbf{47.2} & \textbf{68.5} & \textbf{51.5} & \textbf{31.8} & \textbf{53.0} & \textbf{61.4} \\
\midrule
Swin-T-PolaFormer~\cite{meng2025polaformer} & 44.8 & 67.6 & 49.1 & 40.5 & 64.1 & 43.5 & -- & -- & -- & -- & -- & -- \\
Agent-PVT-L~\cite{Han2023AgentAO} & 46.9 & 69.2 & 51.4 & 42.8 & 66.2 & 46.2 & -- & -- & -- & -- & -- & -- \\
SOFT++-Large~\cite{Lu2024} & 47.0 & 68.3 & 51.7 & 42.2 & 65.2 & 45.4 & 47.0 & 67.8 & 50.4 & 30.2 & 50.9 & 62.0 \\
\rowcolor{blue!20}
LaplacianFormer-Large & \textbf{48.2} & \textbf{70.5} & \textbf{53.0} & \textbf{43.8} & \textbf{67.1} & \textbf{47.4} & \textbf{48.5} & \textbf{69.3} & \textbf{52.4} & \textbf{32.6} & \textbf{52.3} & \textbf{63.8} \\
\bottomrule
\end{tabular}

\label{tab:backbone_comparison}
\end{table}


\subsection{Ablation Studies}
\label{sec:ablation}

\textbf{Convergence Under Varying Condition Numbers.}
We evaluate solver convergence across varying condition numbers. We measure the relative error for Newton--Schulz using the Frobenius norm \( \|X_k - W^{\dagger}\|_F / \|W^{\dagger}\|_F \), and for CG using the Euclidean norm \( \|x_k - x^*\|_2 / \|x^*\|_2 \).As shown in Figure~\ref{fig:convergence-condition}, CG converges more rapidly under well-conditioned settings (e.g., \( \kappa = 2 \)) but degrades significantly as the condition number increases. In contrast, Newton--Schulz exhibits an initial warm-up phase followed by stable convergence even under ill-conditioned regimes (e.g., \( \kappa = 50 \)), indicating greater robustness in practice.
\begin{figure}[H]
  \centering
  \includegraphics[width=1.0\linewidth]{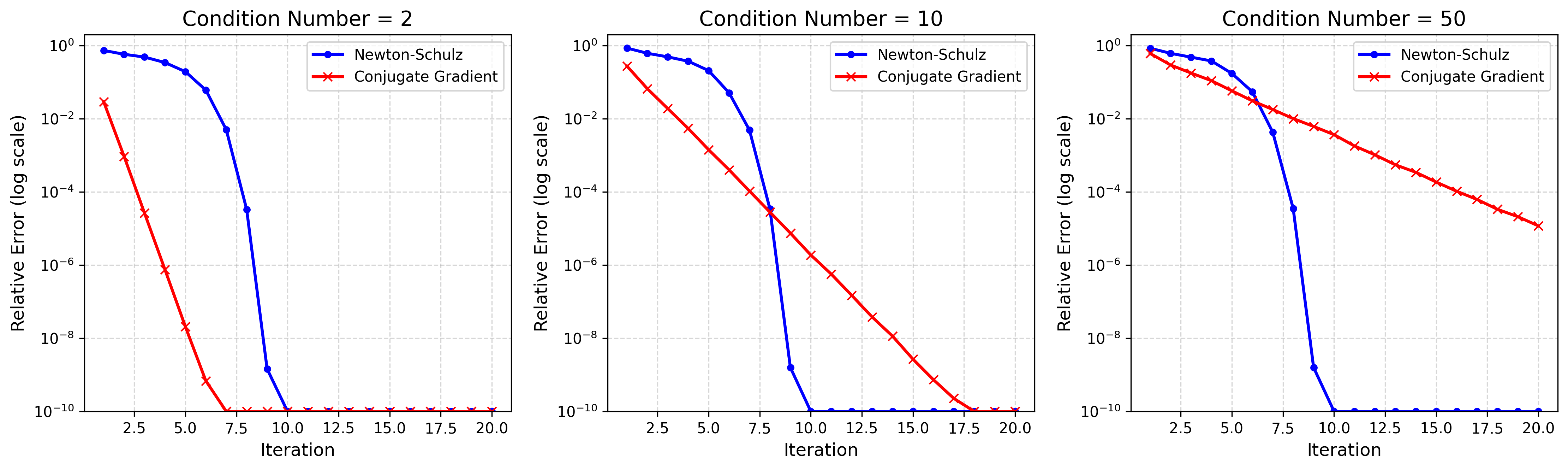}
  \caption{Convergence behavior of Newton--Schulz and conjugate gradient methods under varying condition numbers. Each plot shows relative error (log scale) vs. iteration.}
  \label{fig:convergence-condition}
\end{figure}


\textbf{Inverse Solver Effect.} We compare two iterative solvers—Newton–Schulz (used in our model) and conjugate gradient (CG)—for computing the kernel inverse in linear attention, both implemented with custom CUDA kernels. As shown in Table~\ref{tab:ablation-inverse}, Newton–Schulz achieves higher Top-1 accuracy than CG for both LaplacianFormer-Tiny (81.4\% vs. 79.2\%) and LaplacianFormer-Small (83.8\% vs. 81.4\%), likely due to better GPU convergence and numerical stability.

\textbf{Effect of Laplacian Kernel Scale.}
We study the impact of the Laplacian kernel scale $\lambda$ in the similarity function $\text{sim}_{\text{Lap}}(q, k) = \exp\left( -\frac{\|q - k\|_1}{\lambda} \right)$. As shown in Table~\ref{tab:lambda-effect}, the model achieves the best Top-1 accuracy (81.4\%) when $\lambda = 4$. Small $\lambda$ values (e.g., 0.5, 1) overly suppress long-range interactions, while large values (e.g., 8) yield overly smooth attention, diluting local detail. An intermediate scale ($\lambda = 4$) balances local sensitivity and global context, and is thus fixed in all experiments. Attention map visualizations (Figure~\ref{fig:lap-kernel-vis}) further validate this choice.

\begin{table}[H]
\caption{\footnotesize \textbf{Ablation studies on LaplacianFormer architecture.} \textit{(left)} Top-1 accuracy (\%) of LaplacianFormer variants using different inverse solvers: conjugate gradient (CG) vs. Newton--Schulz (NS).\textit{(right)} Effects of the Laplacian kernel scale $\lambda$ on LaplacianFormer-Tiny.}

\label{tab:ablation}
\centering

\begin{minipage}{0.48\textwidth}
    \centering
    \normalsize
    \label{tab:ablation-inverse}
    \begin{tabular}{@{}lcc@{}}
    \toprule
    Model & CG (\%) & NS (\%) \\
    \midrule
    LaplacianFormer-Tiny  & 79.2 & \textbf{81.4} \\
    LaplacianFormer-Small & 81.4 & \textbf{83.8} \\
    \bottomrule
    \end{tabular}
\end{minipage}
\hfill
\begin{minipage}{0.48\textwidth}
    \centering
    \normalsize
    \label{tab:lambda-effect}
    \begin{tabular}{@{}lccccc@{}}
    \toprule
    $\lambda$ & 0.5 & 1 & 2 & 4 & 8 \\
    \midrule
    Top-1 Acc (\%) $\uparrow$ & 79.4 & 79.6 & 80.1 & \textbf{81.4} & 78.5 \\
    \bottomrule
    \end{tabular}
\end{minipage}
\end{table}

\begin{figure}[H]
\centering
\includegraphics[width=0.7\linewidth]{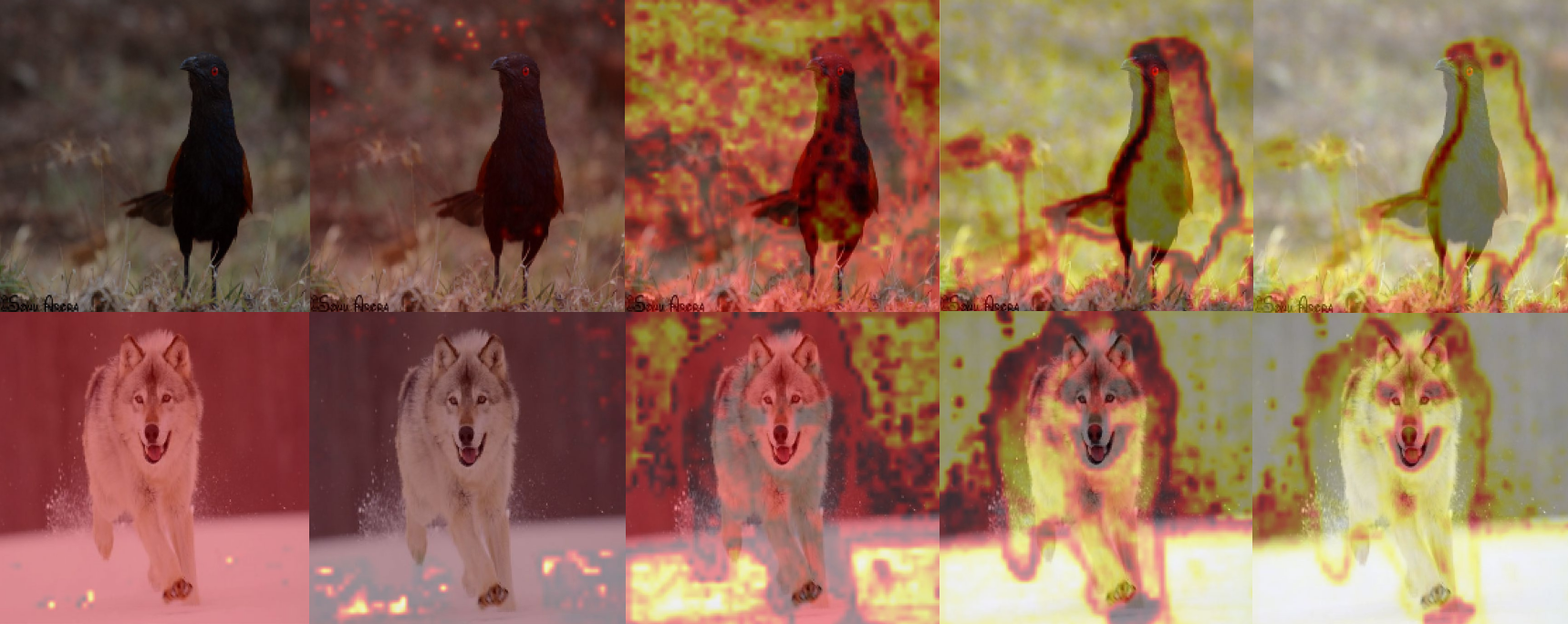}
\caption{\textbf{Visualization of attention maps under different Laplacian kernel scales $\lambda$.} 
From left to right: $\lambda = 0.5$, 1, 2, 4, 8. }
\label{fig:lap-kernel-vis}
\end{figure}

\section{Conclusions and Future Work}
We propose \textbf{LaplacianFormer}, a Transformer variant that employs a Laplacian kernel to construct injective and normalized attention, enabling fine-grained token discrimination with linear complexity. To ensure scalability, we adopt the Nyström approximation and accelerate computation via Newton--Schulz iteration, with efficient CUDA support for both forward and backward passes. LaplacianFormer strikes a balance between expressiveness and efficiency, performing well on both vision and long-sequence tasks. Moreover, it achieves strong results on downstream applications such as object detection and segmentation, further demonstrating its generalization capability.

This work specifically focuses on comparing Laplacian and Gaussian kernels—the latter being the dominant choice in prior linear attention models~\cite{Katharopoulos2020TransformersAR,Lu2021SOFTST,Chen2021SkyformerRS}. Our goal is to challenge this convention through both theoretical analysis and empirical validation. Broader comparisons with other kernel families (e.g., cosine, polynomial) are left as future work.

\section*{Acknowledgments}
This work was supported by Beijing Natural Science Foundation (L231013, L241056, JQ23014), National Natural Science Foundation of China (Nos. 62376271, U22B2034, 62572059, and 62365014),  Jiangxi Provincial Natural Science Foundation (No. 20253BAC280104), Shenzhen S\&T Programme (No. CJGJZD20240729141906008), Taishan Scholars Program No.TSQN202507241, Key R\&D Program of Shandong Province, China No.2025KJHZ013, Shandong Provincial University Youth Innovation and Technology Support Program No.2022KJ291, Shandong Provincial Natural Science Foundation for Young Scholars Program No.ZR2025QC1627, Qilu University of Technology (Shandong Academy of Sciences) Youth Outstanding Talent Program No. 2024QZJH02, Open Project of Key Laboratory of Computing Power Network and Information Security (No. 2024PY021), and the Open Project Program of State Key Laboratory of Virtual Reality Technology and Systems, Beihang University (No. VRLAB2025B03).

\bibliography{iclr2026_conference}

@inproceedings{zhanghivit,
  title={HiViT: A Simpler and More Efficient Design of Hierarchical Vision Transformer},
  author={Zhang, Xiaosong and Tian, Yunjie and Xie, Lingxi and Huang, Wei and Dai, Qi and Ye, Qixiang and Tian, Qi},
  booktitle={The Eleventh International Conference on Learning Representations},
  year={2023},
}

@inproceedings{Zhang2024TheH,
  title={The Hedgehog \& the Porcupine: Expressive Linear Attentions with Softmax Mimicry},
  author={Michael Zhang and Kush S. Bhatia and Hermann Kumbong and Christopher R'e},
  booktitle={The Twelfth International Conference on Learning Representations},
  year={2024},
}

@inproceedings{
  zhen2022cosformer,
  title={cosFormer: Rethinking Softmax In Attention},
  author={Zhen Qin and Weixuan Sun and Hui Deng and Dongxu Li and Yunshen Wei and Baohong Lv and Junjie Yan and Lingpeng Kong and Yiran Zhong},
  booktitle={The Tenth International Conference on Learning Representations},
  year={2022},
}

@inproceedings{
meng2025polaformer,
title={PolaFormer: Polarity-aware Linear Attention for Vision Transformers},
author={Weikang Meng and Yadan Luo and Xin Li and Dongmei Jiang and Zheng Zhang},
booktitle={The Thirteenth International Conference on Learning Representations},
year={2025},
}

@article{Duan2024VisionRWKVEA,
  title={Vision-RWKV: Efficient and Scalable Visual Perception with RWKV-Like Architectures},
  author={Yuchen Duan and Weiyun Wang and Zhe Chen and Xizhou Zhu and Lewei Lu and Tong Lu and Yu Qiao and Hongsheng Li and Jifeng Dai and Wenhai Wang},
  booktitle={The Thirteenth International Conference on Learning Representations},
  year={2025},
}

@article{Kim2024LearningCS,
  title={Learning Correlation Structures for Vision Transformers},
  author={Manjin Kim and Paul Hongsuck Seo and Cordelia Schmid and Minsu Cho},
  journal={2024 IEEE/CVF Conference on Computer Vision and Pattern Recognition (CVPR)},
  year={2024},
  pages={18941-18951},
}

@inproceedings{Katharopoulos2020TransformersAR,
  title={Transformers are RNNs: Fast Autoregressive Transformers with Linear Attention},
  author={Angelos Katharopoulos and Apoorv Vyas and Nikolaos Pappas and Franccois Fleuret},
  booktitle={International Conference on Machine Learning},
  year={2020},
}

@inproceedings{Hassani2024FasterNA,
 author = {Hassani, Ali and Hwu, Wen-mei and Shi, Humphrey},
 booktitle = {Advances in Neural Information Processing Systems},
 editor = {A. Globerson and L. Mackey and D. Belgrave and A. Fan and U. Paquet and J. Tomczak and C. Zhang},
 pages = {64717--64734},
 title = {Faster Neighborhood Attention: Reducing the $O(n^2)$ Cost of Self Attention at the Threadblock Level},
 volume = {37},
 year = {2024}
}

@article{Su2024ScanFormerRE,
  title={ScanFormer: Referring Expression Comprehension by Iteratively Scanning},
  author={Wei Su and Peihan Miao and Huanzhang Dou and Xi Li},
  journal={2024 IEEE/CVF Conference on Computer Vision and Pattern Recognition (CVPR)},
  year={2024},
  pages={13449-13458},
}

@inproceedings{Hou2024ProTransformerRT,
title={ProTransformer: Robustify Transformers via Plug-and-Play Paradigm},
author={Zhichao Hou and Weizhi Gao and Yuchen Shen and Xiaorui Liu},
booktitle={ICLR 2024 Workshop on Reliable and Responsible Foundation Models},
year={2024},
}

@InProceedings{Yu2024EmbeddingFreeTW,
author="Yu, Hyunwoo
and Cho, Yubin
and Kang, Beoungwoo
and Moon, Seunghun
and Kong, Kyeongbo
and Kang, Suk-Ju",
editor="Leonardis, Ale{\v{s}}
and Ricci, Elisa
and Roth, Stefan
and Russakovsky, Olga
and Sattler, Torsten
and Varol, G{\"u}l",
title="Embedding-Free Transformer with Inference Spatial Reduction for Efficient Semantic Segmentation",
booktitle="Computer Vision -- ECCV 2024",
year="2025",
pages="92--110",

}

@article{Jiang2024RBSFormerET,
  title={RBSFormer: Enhanced Transformer Network for Raw Image Super-Resolution},
  author={Siyuan Jiang and Senyan Xu and Xingfu Wang},
  journal={2024 IEEE/CVF Conference on Computer Vision and Pattern Recognition Workshops (CVPRW)},
  year={2024},
  pages={6479-6488},
}

@inproceedings{Keles2022OnTC,
  title={On The Computational Complexity of Self-Attention},
  author={Feyza Duman Keles and Pruthuvi Maheshakya Wijewardena and Chinmay Hegde},
  booktitle={International Conference on Algorithmic Learning Theory},
  year={2022},
}

@inproceedings{Vaswani2017AttentionIA,
  title={Attention is All you Need},
  author={Ashish Vaswani and Noam M. Shazeer and Niki Parmar and Jakob Uszkoreit and Llion Jones and Aidan N. Gomez and Lukasz Kaiser and Illia Polosukhin},
  booktitle={Neural Information Processing Systems},
  year={2017},
}

@inproceedings{Dosovitskiy2020AnII,
  author       = {Alexey Dosovitskiy and
                  Lucas Beyer and
                  Alexander Kolesnikov and
                  Dirk Weissenborn and
                  Xiaohua Zhai and
                  Thomas Unterthiner and
                  Mostafa Dehghani and
                  Matthias Minderer and
                  Georg Heigold and
                  Sylvain Gelly and
                  Jakob Uszkoreit and
                  Neil Houlsby},
  title        = {An Image is Worth 16x16 Words: Transformers for Image Recognition
                  at Scale},
  booktitle    = {9th International Conference on Learning Representations, {ICLR} 2021,
                  Virtual Event, Austria, May 3-7, 2021},
  year         = {2021},
}

@article{Wang2021PyramidVT,
  title={Pyramid Vision Transformer: A Versatile Backbone for Dense Prediction without Convolutions},
  author={Wenhai Wang and Enze Xie and Xiang Li and Deng-Ping Fan and Kaitao Song and Ding Liang and Tong Lu and Ping Luo and Ling Shao},
  journal={2021 IEEE/CVF International Conference on Computer Vision (ICCV)},
  year={2021},
  pages={548-558},
}

@article{Wang2021PVTVI,
  title={PVT v2: Improved baselines with Pyramid Vision Transformer},
  author={Wenhai Wang and Enze Xie and Xiang Li and Deng-Ping Fan and Kaitao Song and Ding Liang and Tong Lu and Ping Luo and Ling Shao},
  journal={Computational Visual Media},
  year={2021},
  volume={8},
  pages={415 - 424},
}

@inproceedings{Touvron2020TrainingDI,
  title={Training data-efficient image transformers \& distillation through attention},
  author={Hugo Touvron and Matthieu Cord and Matthijs Douze and Francisco Massa and Alexandre Sablayrolles and Herv'e J'egou},
  booktitle={International Conference on Machine Learning},
  year={2020},
}

@article{Touvron2021GoingDW,
  title={Going deeper with Image Transformers},
  author={Hugo Touvron and Matthieu Cord and Alexandre Sablayrolles and Gabriel Synnaeve and Herv'e J'egou},
  journal={2021 IEEE/CVF International Conference on Computer Vision (ICCV)},
  year={2021},
  pages={32-42},
}

@article{Liu2021SwinTH,
  title={Swin Transformer: Hierarchical Vision Transformer using Shifted Windows},
  author={Ze Liu and Yutong Lin and Yue Cao and Han Hu and Yixuan Wei and Zheng Zhang and Stephen Lin and Baining Guo},
  journal={2021 IEEE/CVF International Conference on Computer Vision (ICCV)},
  year={2021},
  pages={9992-10002},
}

@article{Liu2021SwinTV,
  title={Swin Transformer V2: Scaling Up Capacity and Resolution},
  author={Ze Liu and Han Hu and Yutong Lin and Zhuliang Yao and Zhenda Xie and Yixuan Wei and Jia Ning and Yue Cao and Zheng Zhang and Li Dong and Furu Wei and Baining Guo},
  journal={2022 IEEE/CVF Conference on Computer Vision and Pattern Recognition (CVPR)},
  year={2021},
  pages={11999-12009},
}

@inproceedings{Touvron2022DeiTIR,
  title={DeiT III: Revenge of the ViT},
  author={Hugo Touvron and Matthieu Cord and Herv'e J'egou},
  booktitle={European Conference on Computer Vision},
  year={2022},
}

@inproceedings{Zhu2020DeformableDD,
  author       = {Xizhou Zhu and
                  Weijie Su and
                  Lewei Lu and
                  Bin Li and
                  Xiaogang Wang and
                  Jifeng Dai},
  title        = {Deformable {DETR:} Deformable Transformers for End-to-End Object Detection},
  booktitle    = {9th International Conference on Learning Representations, {ICLR} 2021,
                  Virtual Event, Austria, May 3-7, 2021},
  year         = {2021},
}

@inproceedings{
Zhang2022DINODW,
title={{DINO}: {DETR} with Improved DeNoising Anchor Boxes for End-to-End Object Detection},
author={Hao Zhang and Feng Li and Shilong Liu and Lei Zhang and Hang Su and Jun Zhu and Lionel Ni and Heung-Yeung Shum},
booktitle={The Eleventh International Conference on Learning Representations },
year={2023},
}

@article{Zheng2020RethinkingSS,
  title={Rethinking Semantic Segmentation from a Sequence-to-Sequence Perspective with Transformers},
  author={Sixiao Zheng and Jiachen Lu and Hengshuang Zhao and Xiatian Zhu and Zekun Luo and Yabiao Wang and Yanwei Fu and Jianfeng Feng and Tao Xiang and Philip H. S. Torr and Li Zhang},
  journal={2021 IEEE/CVF Conference on Computer Vision and Pattern Recognition (CVPR)},
  year={2020},
  pages={6877-6886},
}

@inproceedings{Xie2021SegFormerSA,
  title={SegFormer: Simple and Efficient Design for Semantic Segmentation with Transformers},
  author={Enze Xie and Wenhai Wang and Zhiding Yu and Anima Anandkumar and Jos{\'e} Manuel {\'A}lvarez and Ping Luo},
  booktitle={Neural Information Processing Systems},
  year={2021},
}

@inproceedings{Cheng2021PerPixelCI,
  title={Per-Pixel Classification is Not All You Need for Semantic Segmentation},
  author={Bowen Cheng and Alexander G. Schwing and Alexander Kirillov},
  booktitle={Neural Information Processing Systems},
  year={2021},
}

@article{Lu2021SOFTST,
  title={Soft: Softmax-free transformer with linear complexity},
  author={Lu, Jiachen and Yao, Jinghan and Zhang, Junge and Zhu, Xiatian and Xu, Hang and Gao, Weiguo and Xu, Chunjing and Xiang, Tao and Zhang, Li},
  journal={Advances in Neural Information Processing Systems},
  volume={34},
  pages={21297--21309},
  year={2021}
}

@inproceedings{Han2023AgentAO,
  title={Agent attention: On the integration of softmax and linear attention},
  author={Han, Dongchen and Ye, Tianzhu and Han, Yizeng and Xia, Zhuofan and Pan, Siyuan and Wan, Pengfei and Song, Shiji and Huang, Gao},
  booktitle={European Conference on Computer Vision},
  pages={124--140},
  year={2024},
  organization={Springer}
}

@inproceedings{han2024inline,
  title={Bridging the Divide: Reconsidering Softmax and Linear Attention
},
  author={Han, Dongchen and Pu, Yifan and Xia, Zhuofan and Han, Yizeng and Pan, Xuran and Li, Xiu and Lu, Jiwen and Song, Shiji and Huang, Gao},
  booktitle={NeurIPS},
  year={2024},
}

@inproceedings{Williams2000UsingTN,
  title={Using the Nystr{\"o}m Method to Speed Up Kernel Machines},
  author={Christopher K. I. Williams and Matthias W. Seeger},
  booktitle={Neural Information Processing Systems},
  year={2000},
}

@article{Deng2009ImageNetAL,
  title={ImageNet: A large-scale hierarchical image database},
  author={Jia Deng and Wei Dong and Richard Socher and Li-Jia Li and K. Li and Li Fei-Fei},
  journal={2009 IEEE Conference on Computer Vision and Pattern Recognition},
  year={2009},
  pages={248-255},
}

@inproceedings{
choromanski2021rethinking,
title={Rethinking Attention with Performers},
author={Krzysztof Marcin Choromanski and Valerii Likhosherstov and David Dohan and Xingyou Song and Andreea Gane and Tamas Sarlos and Peter Hawkins and Jared Quincy Davis and Afroz Mohiuddin and Lukasz Kaiser and David Benjamin Belanger and Lucy J Colwell and Adrian Weller},
booktitle={International Conference on Learning Representations},
year={2021},
}

@article{Xiong2021NystrmformerAN,
  title={Nystr{\"o}mformer: A Nystr{\"o}m-Based Algorithm for Approximating Self-Attention},
  author={Yunyang Xiong and Zhanpeng Zeng and Rudrasis Chakraborty and Mingxing Tan and Glenn Moo Fung and Yin Li and Vikas Singh},
  journal={Proceedings of the ... AAAI Conference on Artificial Intelligence. AAAI Conference on Artificial Intelligence},
  year={2021},
  volume={35 16},
  pages={
          14138-14148
        },
}

@article{paszke2017automatic,
  title={Automatic differentiation in PyTorch},
  author={Paszke, Adam and Gross, Sam and Chintala, Soumith and Chanan, Gregory and Yang, Edward and DeVito, Zachary and Lin, Zeming and Desmaison, Alban and Antiga, Luca and Lerer, Adam},
  year={2017}
}

@article{SU2024127063,
title = {RoFormer: Enhanced transformer with Rotary Position Embedding},
journal = {Neurocomputing},
volume = {568},
pages = {127063},
year = {2024},
issn = {0925-2312},
doi = {https://doi.org/10.1016/j.neucom.2023.127063},
author = {Jianlin Su and Murtadha Ahmed and Yu Lu and Shengfeng Pan and Wen Bo and Yunfeng Liu}
}

@article{Han2023FLattenTV,
  title={FLatten Transformer: Vision Transformer using Focused Linear Attention},
  author={Dongchen Han and Xuran Pan and Yizeng Han and Shiji Song and Gao Huang},
  journal={2023 IEEE/CVF International Conference on Computer Vision (ICCV)},
  year={2023},
  pages={5938-5948},
}

@inproceedings{Guo2024SLABET,
  title={SLAB: Efficient Transformers with Simplified Linear Attention and Progressive Re-parameterized Batch Normalization},
  author={Guo, Jialong and Chen, Xinghao and Tang, Yehui  and Wang, Yunhe},
  booktitle={International Conference on Machine Learning},
  year={2024}
}

@article{Lu2024,
  author       = {Jiachen Lu and Junge Zhang and Xiatian Zhu and Jianfeng Feng and Tao Xiang and Li Zhang},
  title        = {Softmax-Free Linear Transformers},
  journal      = {International Journal of Computer Vision},
  volume       = {132},
  number       = {8},
  pages        = {3355--3374},
  year         = {2024},
  month        = aug,
  doi          = {10.1007/s11263-024-02035-5},
  issn         = {1573-1405}
}

@article{He2017MaskR,
  title={Mask R-CNN},
  author={Kaiming He and Georgia Gkioxari and Piotr Doll{\'a}r and Ross B. Girshick},
  journal={2017 IEEE International Conference on Computer Vision (ICCV)},
  year={2017},
  pages={2980-2988},

}

@article{Lin2017FocalLF,
  title={Focal Loss for Dense Object Detection},
  author={Tsung-Yi Lin and Priya Goyal and Ross B. Girshick and Kaiming He and Piotr Doll{\'a}r},
  journal={2017 IEEE International Conference on Computer Vision (ICCV)},
  year={2017},
  pages={2999-3007},
}

@inproceedings{Chen2021SkyformerRS,
  title={Skyformer: Remodel Self-Attention with Gaussian Kernel and Nystr{\"o}m Method},
  author={Yifan Chen and Qi Zeng and Heng Ji and Yun Yang},
  booktitle={Neural Information Processing Systems},
  year={2021},
}

@article{Bui2025RevisitingKA,
  title={Revisiting Kernel Attention with Correlated Gaussian Process Representation},
  author={Long Minh Bui and Tho Tran Huu and Duy Dinh and Tan Minh Nguyen and Trong Nghia Hoang},
  journal={ArXiv},
  year={2025},
  volume={abs/2502.20525},
}

@article{Kashiwagi2021GaussianKS,
  title={Gaussian Kernelized Self-Attention for Long Sequence Data and its Application to CTC-Based Speech Recognition},
  author={Yosuke Kashiwagi and Emiru Tsunoo and Shinji Watanabe},
  journal={ICASSP 2021 - 2021 IEEE International Conference on Acoustics, Speech and Signal Processing (ICASSP)},
  year={2021},
  pages={6214-6218},
}

@inproceedings{han2024demystify,
  title={Demystify Mamba in Vision: A Linear Attention Perspective},
  author={Han, Dongchen and Wang, Ziyi and Xia, Zhuofan and Han, Yizeng and Pu, Yifan and Ge, Chunjiang and Song, Jun and Song, Shiji and Zheng, Bo and Huang, Gao},
  booktitle={NeurIPS},
  year={2024},
}

@inproceedings{huang2023stvit,
  title     = {Vision Transformer with Super Token Sampling},
  author    = {Huang, Huaibo and Zhou, Xiaoqiang and Cao, Jie and He, Ran and Tan, Tieniu},
  booktitle = {Proceedings of the IEEE/CVF Conference on Computer Vision and Pattern Recognition (CVPR)},
  year      = {2023},
  pages     = {12703--12712}
}

@Article{zhu2023biformer,
  author  = {Lei Zhu and Xinjiang Wang and Zhanghan Ke and Wayne Zhang and Rynson Lau},
  title   = {BiFormer: Vision Transformer with Bi-Level Routing Attention},
  journal = {Proceedings of the IEEE/CVF Conference on Computer Vision and Pattern Recognition (CVPR)},
  year    = {2023},
}

@inproceedings{iclr2024MogaNet,
  title={MogaNet: Multi-order Gated Aggregation Network},
  author={Siyuan Li and Zedong Wang and Zicheng Liu and Cheng Tan and Haitao Lin and Di Wu and Zhiyuan Chen and Jiangbin Zheng and Stan Z. Li},
  booktitle={International Conference on Learning Representations},
  year={2024}
}

@inproceedings{hassani2023neighborhood,
	title        = {Neighborhood Attention Transformer},
	author       = {Ali Hassani and Steven Walton and Jiachen Li and Shen Li and Humphrey Shi},
	booktitle    = {Proceedings of the IEEE/CVF Conference on Computer Vision and Pattern Recognition (CVPR)},
	month        = {June},
	year         = {2023},
	pages        = {6185-6194}
}

@article{Zhang2021DeepLL,
  title={Deep Long-Tailed Learning: A Survey},
  author={Yifan Zhang and Bingyi Kang and Bryan Hooi and Shuicheng Yan and Jiashi Feng},
  journal={IEEE Transactions on Pattern Analysis and Machine Intelligence},
  year={2021},
  volume={45},
  pages={10795-10816},
}
\bibliographystyle{iclr2026_conference}

\clearpage
\appendix
\section{Gradient Behavior of Gaussian vs. Laplacian Kernels}

\begin{proposition}[Gradient Magnitude Decay in High Dimensions]
Let \( \mathbf{x}, \mathbf{y} \in \mathbb{R}^d \). As the distance \( \|\mathbf{x} - \mathbf{y}\| \to 0 \), the gradient magnitude of the Gaussian kernel decays linearly to zero. In contrast, the gradient magnitude of the Laplacian kernel remains at a constant non-zero order proportional to \( \sqrt{d} \).
\end{proposition}

\begin{proof}
Let \( \mathbf{t} = \mathbf{x} - \mathbf{y} \in \mathbb{R}^d \). We analyze the \( \ell_2 \) norm of the gradient with respect to \( \mathbf{x} \), denoted as \( \|\nabla_{\mathbf{x}} k(\mathbf{x}, \mathbf{y})\|_2 \).

\textbf{1. Gaussian kernel} \\
The multi-dimensional Gaussian kernel is defined as:
\[
k_{\text{Gauss}}(\mathbf{x}, \mathbf{y}) = \exp\left(-\frac{\|\mathbf{t}\|_2^2}{2\sigma^2}\right)
\]
The gradient vector with respect to \( \mathbf{x} \) is:
\[
\nabla_{\mathbf{x}} k_{\text{Gauss}} = -\frac{\mathbf{t}}{\sigma^2} \exp\left(-\frac{\|\mathbf{t}\|_2^2}{2\sigma^2}\right)
\]
Taking the \( \ell_2 \) norm of the gradient yields:
\[
\|\nabla_{\mathbf{x}} k_{\text{Gauss}}\|_2 = \frac{\|\mathbf{t}\|_2}{\sigma^2} \exp\left(-\frac{\|\mathbf{t}\|_2^2}{2\sigma^2}\right)
\]
As \( \|\mathbf{t}\|_2 \to 0 \), the exponential term approaches \( 1 \), resulting in:
\[
\|\nabla_{\mathbf{x}} k_{\text{Gauss}}\|_2 \sim \frac{\|\mathbf{t}\|_2}{\sigma^2}
\]
This confirms that the gradient magnitude of the Gaussian kernel diminishes \textbf{linearly} to zero as \( \mathbf{x} \to \mathbf{y} \).

\textbf{2. Laplacian kernel} \\
The multi-dimensional Laplacian kernel is defined as:
\[
k_{\text{Laplace}}(\mathbf{x}, \mathbf{y}) = \exp\left(-\frac{\|\mathbf{t}\|_1}{\lambda}\right)
\]
For \( t_i \ne 0 \) (\( \forall i \in \{1, \dots, d\} \)), the gradient vector is:
\[
\nabla_{\mathbf{x}} k_{\text{Laplace}} = -\frac{1}{\lambda} \text{sign}(\mathbf{t}) \exp\left(-\frac{\|\mathbf{t}\|_1}{\lambda}\right)
\]
where \( \text{sign}(\mathbf{t}) \in \{-1, 1\}^d \) is applied element-wise. The \( \ell_2 \) norm of this sign vector is \( \|\text{sign}(\mathbf{t})\|_2 = \sqrt{d} \). Taking the \( \ell_2 \) norm of the gradient yields:
\[
\|\nabla_{\mathbf{x}} k_{\text{Laplace}}\|_2 = \frac{\sqrt{d}}{\lambda} \exp\left(-\frac{\|\mathbf{t}\|_1}{\lambda}\right)
\]
As \( \mathbf{t} \to \mathbf{0} \), the exponential term approaches \( 1 \), resulting in:
\[
\|\nabla_{\mathbf{x}} k_{\text{Laplace}}\|_2 \sim \frac{\sqrt{d}}{\lambda}
\]
Thus, in a \( d \)-dimensional space, the gradient magnitude of the Laplacian kernel is bounded away from zero and remains at a \textbf{constant order} \( \sqrt{d}/\lambda \). 
\end{proof}

\section{Mathematical Proof}

This section offers mathematical proofs for the propositions outlined in the main paper.

\begin{proposition}[Injectivity of Laplacian-based Kernel Embedding]
Let $\mathbf{q}_i, \mathbf{q}_j \in \mathbb{R}^d$, and define their kernel similarity vectors as
\[
\mathbf{g}_i = \left[k(\mathbf{q}_i, \mathbf{k}_1), \ldots, k(\mathbf{q}_i, \mathbf{k}_N)\right]^\top,
\]
where $k(\mathbf{q}, \mathbf{k}) = \exp\left(-\frac{\|\mathbf{q} - \mathbf{k}\|_1}{\lambda}\right)$ is the Laplacian kernel, and $\{\mathbf{k}_1, \ldots, \mathbf{k}_N\} \subset \mathbb{R}^d$ is a fixed set of anchor points. Define the normalized feature mapping
\[
\mathbf{z}_i = \boldsymbol{\Sigma}^{-1/2} \left( \mathbf{g}_i - \frac{1}{N} \sum_{k=1}^N g_i^k \cdot \mathbf{1} \right) + \frac{1}{N} \mathbf{1},
\]
where $\boldsymbol{\Sigma}$ is the empirical covariance matrix over $\{\mathbf{g}_i\}$. Suppose the key set is sufficiently rich such that $\mathbf{q}_i \neq \mathbf{q}_j \Rightarrow \mathbf{g}_i \neq \mathbf{g}_j$. Assume further that for any $\mathbf{q}_i \neq \mathbf{q}_j$ and $c \in \mathbb{R}$, $\mathbf{g}_i - \mathbf{g}_j \neq c\mathbf{1}$ (no constant-vector degeneracy). Then the mapping $\mathbf{q}_i \mapsto \mathbf{z}_i \in \mathbb{R}^N$ is injective.
\end{proposition}

\begin{proof}
We proceed by contradiction. Assume that $\mathbf{q}_i \neq \mathbf{q}_j$ but $\mathbf{z}_i = \mathbf{z}_j$.

For notational convenience, let $\mu_i = \frac{1}{N} \sum_{k=1}^N g_i^k$ denote the mean of the components of $\mathbf{g}_i$. We can then rewrite the normalized feature mapping as:
\begin{equation*}
    \mathbf{z}_i = \boldsymbol{\Sigma}^{-1/2}(\mathbf{g}_i - \mu_i\mathbf{1}) + \frac{1}{N}\mathbf{1}.
\end{equation*}

Expanding the assumed equality $\mathbf{z}_i = \mathbf{z}_j$, we obtain:
\begin{equation*}
    \boldsymbol{\Sigma}^{-1/2}(\mathbf{g}_i - \mu_i\mathbf{1}) + \frac{1}{N}\mathbf{1} = \boldsymbol{\Sigma}^{-1/2}(\mathbf{g}_j - \mu_j\mathbf{1}) + \frac{1}{N}\mathbf{1}.
\end{equation*}

Subtracting the constant vector $\frac{1}{N}\mathbf{1}$ from both sides yields:
\begin{equation*}
    \boldsymbol{\Sigma}^{-1/2}(\mathbf{g}_i - \mu_i\mathbf{1}) = \boldsymbol{\Sigma}^{-1/2}(\mathbf{g}_j - \mu_j\mathbf{1}).
\end{equation*}

Assuming non-degenerate cases where the covariance matrix $\boldsymbol{\Sigma}$ is strictly positive definite, $\boldsymbol{\Sigma}^{-1/2}$ is invertible. Left-multiplying both sides by $\boldsymbol{\Sigma}^{1/2}$ gives:
\begin{equation*}
    \mathbf{g}_i - \mu_i\mathbf{1} = \mathbf{g}_j - \mu_j\mathbf{1}.
\end{equation*}

Rearranging the terms gives:
\begin{equation*}
    \mathbf{g}_i - \mathbf{g}_j = (\mu_i - \mu_j)\mathbf{1}.
\end{equation*}

Let $c = \mu_i - \mu_j$, where $c \in \mathbb{R}$ is a scalar. Then the equation becomes:
\begin{equation} \label{eq:contradiction_base}
    \mathbf{g}_i - \mathbf{g}_j = c\mathbf{1}.
\end{equation}

We now consider two mutually exclusive cases for the value of $c$:

\textbf{Case 1: $c = 0$.} \\
Equation (\ref{eq:contradiction_base}) implies $\mathbf{g}_i = \mathbf{g}_j$. However, since we assumed $\mathbf{q}_i \neq \mathbf{q}_j$, this directly contradicts the proposition's assumption that $\mathbf{q}_i \neq \mathbf{q}_j \Rightarrow \mathbf{g}_i \neq \mathbf{g}_j$.

\textbf{Case 2: $c \neq 0$.} \\
Equation (\ref{eq:contradiction_base}) implies $\mathbf{g}_i - \mathbf{g}_j = c\mathbf{1}$ for some scalar $c$. Since $\mathbf{q}_i \neq \mathbf{q}_j$, this directly contradicts the assumption of no constant-vector degeneracy.

In both cases, we arrive at a contradiction. Therefore, our initial assumption must be false. The only possibility is that $\mathbf{z}_i = \mathbf{z}_j \Rightarrow \mathbf{q}_i = \mathbf{q}_j$. Thus, the mapping is injective.
\end{proof}

\begin{proposition}[Linear-Time Computation of Laplacian Feature Map via Nyström Approximation]
Let \( \mathbf{q}_i \in \mathbb{R}^d \) be a query and \( \{ \tilde{\mathbf{k}}_1, \dots, \tilde{\mathbf{k}}_m \} \subset \mathbb{R}^d \) be a set of $m$ landmark keys sampled via the Nyström method, where $m \ll N$. Define the Laplacian kernel \( k(\mathbf{q}, \mathbf{k}) = \exp\left(-\frac{\|\mathbf{q} - \mathbf{k}\|_1}{\lambda}\right) \), and let
\[
\tilde{\mathbf{g}}_i = \left[k(\mathbf{q}_i, \tilde{\mathbf{k}}_1), \dots, k(\mathbf{q}_i, \tilde{\mathbf{k}}_m)\right]^\top \in \mathbb{R}^m
\]
be the landmark kernel similarity vector for the $i$-th query. Assuming input dimension \( d \) and the number of landmarks \( m \) are constants with respect to the sequence length $N$, the normalized embedding
\begin{equation}
\mathbf{z}_i = \boldsymbol{\Sigma}^{-1/2} \left( \tilde{\mathbf{g}}_i - \frac{1}{m} \sum_{j=1}^m \tilde{g}_i^j \cdot \mathbf{1} \right) + \frac{1}{m} \mathbf{1}
\label{eq:laplace-feature-map-nystrom}
\end{equation}
can be computed in \( \mathcal{O}(m) \) time and space. Consequently, computing the feature maps for all $N$ queries achieves an overall \( \mathcal{O}(N) \) time and space complexity \emph{at inference time}, provided that \( \boldsymbol{\Sigma}^{-1/2} \) is approximated by an $m \times m$ diagonal whitening matrix estimated offline.
\end{proposition}

\begin{proof}
The proof proceeds by analyzing the time and space complexity of computing the sequence of embeddings $\mathbf{z}_1, \dots, \mathbf{z}_N$ at inference time, based on the $m$ sampled landmarks.

\textbf{Step 1: Computation of the landmark similarity vector $\tilde{\mathbf{g}}_i$.} \\
For each landmark key $\tilde{\mathbf{k}}_j$ ($j = 1, \dots, m$), computing the Laplacian kernel element $\tilde{g}_i^j = \exp\left(-\frac{\|\mathbf{q}_i - \tilde{\mathbf{k}}_j\|_1}{\lambda}\right)$ requires:
\begin{enumerate}
    \item Vector subtraction $\mathbf{q}_i - \tilde{\mathbf{k}}_j$ in $\mathbb{R}^d$, which takes $\mathcal{O}(d)$ time.
    \item Computing the $L_1$ norm of the resulting difference vector, taking $\mathcal{O}(d)$ time.
    \item Scalar division by $\lambda$ and exponentiation, taking $\mathcal{O}(1)$ time.
\end{enumerate}
Since the input dimension $d$ is assumed to be constant, computing one element $\tilde{g}_i^j$ is bounded by $\mathcal{O}(d) = \mathcal{O}(1)$. Computing all $m$ elements for the vector $\tilde{\mathbf{g}}_i$ therefore takes $\mathcal{O}(m)$ time. Storing the resulting vector $\tilde{\mathbf{g}}_i$ requires $\mathcal{O}(m)$ space.

\textbf{Step 2: Centering the feature vector.} \\
Let $\mu_i = \frac{1}{m} \sum_{j=1}^m \tilde{g}_i^j$. Summing the $m$ elements of $\tilde{\mathbf{g}}_i$ and dividing by $m$ takes $\mathcal{O}(m)$ time. Subtracting this scalar mean from each component of $\tilde{\mathbf{g}}_i$ to compute the centered vector $\tilde{\mathbf{g}}_i - \mu_i\mathbf{1}$ takes an additional $\mathcal{O}(m)$ time. The space required for the centered vector is $\mathcal{O}(m)$.

\textbf{Step 3: Applying the whitening transformation.} \\
By assumption, the $m \times m$ matrix $\boldsymbol{\Sigma}^{-1/2}$ is approximated by a diagonal matrix computed offline. Let $\mathbf{D} = \text{diag}(d_1, \dots, d_m)$ be this diagonal approximation. The matrix-vector multiplication $\mathbf{D} (\tilde{\mathbf{g}}_i - \mu_i\mathbf{1})$ reduces to an element-wise multiplication (Hadamard product) between the diagonal entries of $\mathbf{D}$ and the centered vector. This element-wise operation takes $\mathcal{O}(m)$ time. Storing the diagonal of $\mathbf{D}$ requires $\mathcal{O}(m)$ space, and the resulting transformed vector also takes $\mathcal{O}(m)$ space. \emph{(Note: Without the diagonal assumption, a full matrix-vector multiplication would require $\mathcal{O}(m^2)$ time)}.

\textbf{Step 4: Adding the final bias term.} \\
Adding the scalar constant $\frac{1}{m}$ to each element of the resulting vector to obtain the final normalized embedding $\mathbf{z}_i \in \mathbb{R}^m$ requires $\mathcal{O}(m)$ time.

\textbf{Conclusion:}\\
For a single query $\mathbf{q}_i$, summing the time complexities yields $\mathcal{O}(m) + \mathcal{O}(m) + \mathcal{O}(m) + \mathcal{O}(m) = \mathcal{O}(m)$ inference time. Extending this operation to the entire sequence of $N$ queries results in a total inference time of $N \times \mathcal{O}(m) = \mathcal{O}(N \cdot m)$. Because the number of Nyström landmarks $m$ is a predefined parameter such that $m \ll N$, $m$ acts as a constant with respect to the sequence length. Thus, the overall computation of the feature maps scales linearly, achieving $\mathcal{O}(N)$ time and space complexity.
\end{proof}

\end{document}